\documentclass[english]{article}
\usepackage[latin9]{inputenc}
\usepackage{geometry}
\usepackage{array}
\usepackage{float}
\usepackage{multirow}
\usepackage{amsmath}
\usepackage{amssymb}
\usepackage{graphicx}
\usepackage{esint}
\usepackage{caption}
\usepackage{amsthm}
\usepackage{changepage}
\usepackage[authoryear]{natbib}
\usepackage[misc,geometry]{ifsym} 

\makeatletter

\newtheorem{theorem}{Theorem}
\newtheorem{definition}{Definition}

\providecommand{\tabularnewline}{\\}
\floatstyle{ruled}
\newfloat{algorithm}{tbp}{loa}
\providecommand{\algorithmname}{Algorithm}
\floatname{algorithm}{\protect\algorithmname}

\@ifundefined{showcaptionsetup}{}{%
 \PassOptionsToPackage{caption=false}{subfig}}
\usepackage{subfig}
\makeatother

\begin{document}

\title{Labeled Directed Acyclic Graphs: a generalization of context-specific independence in directed graphical models}

\author{Johan Pensar$^{\ast,1}$, Henrik Nyman$^{1}$, Timo Koski$^{3}$, Jukka Corander$^{1,2}$ \\
$^{1}$Department of Mathematics, \AA bo Akademi University, Finland \\
$^{2}$Department of Mathematics and statistics, University of Helsinki, Finland \\
$^{3}$Department of Mathematics \\ KTH Royal Institute of Technology, Stockholm, Sweden \\
$^{\ast}$Corresponding author, Email: jopensar@abo.fi}

\date{}
\maketitle
\begin{abstract}
We introduce a novel class of labeled directed acyclic graph (LDAG)
models for finite sets of discrete variables. LDAGs generalize earlier
proposals for allowing local structures in the conditional probability
distribution of a node, such that unrestricted label sets determine
which edges can be deleted from the underlying directed acyclic graph (DAG) for a given context.
Several properties of these models are derived, including a generalization
of the concept of Markov equivalence classes. Efficient
Bayesian learning of LDAGs is enabled by introducing an LDAG-based
factorization of the Dirichlet prior for the model parameters, such
that the marginal likelihood can be calculated analytically. In addition,
we develop a novel prior distribution for the model structures that
can appropriately penalize a model for its labeling complexity.
A non-reversible Markov chain Monte Carlo algorithm combined with
a greedy hill climbing approach is used for illustrating the useful
properties of LDAG models for both real and synthetic data sets.\\


\end{abstract}
\parskip 0pt
\section{Introduction\label{sec:Introduction}}
Directed acyclic graphs have gained widespread popularity as representations
of complex multivariate systems (\citet{key-1,Koller+Friedman:09}).
Despite their advantageous properties for representing dependencies
among variables in a modular fashion, several proposals for making
them more flexible and parsimonious have been presented (\citet{Boutilier96CSIinBN,Friedman96learnBNlocstruct,Chickering97abayesian,Eriksen99CSI,Poole03exploitingcontextual,Koller+Friedman:09}).
In particular, an important notion is to allow the dependencies to have
local structures, such that a node need not explicitly depend on all
the combinations of values of its parents. This leads to context-specific
independence which can substantially reduce the parametric dimensionality
of a network model and lead to a more expressive interpretation of
the dependence structure (\citet{Boutilier96CSIinBN,Friedman96learnBNlocstruct,Poole03exploitingcontextual,Koller+Friedman:09}).
Context-specific independencies have also been seemingly separately
considered for undirected graphical models by multiple authors (\citet{Corander03LGM,Hojsgaard03,Hojsgaard04}).

Here we generalize the context-specific independence models proposed
in \citet{Boutilier96CSIinBN} by allowing the independencies to be represented
in terms of labels for the parental configurations of a node in an unrestricted
manner. This approach thus goes beyond the trees of conditional probability
tables considered in \citet{Boutilier96CSIinBN} as it instead introduces a partition
of the parental configurations into classes with invariant conditional
probability distributions for the outcomes that are assigned to the
same class. It is shown that such a definition leads to a model class
with a number of desirable properties, and we derive several properties
of the models, including their identifiability and an LDAG version
of the concept of a Markov equivalence class. 

We develop an efficient method for Bayesian learning of LDAG models
from a set of training data by introducing a prior distribution for
the model parameters that factorizes in a comparable manner as the
standard Dirichlet distribution used for learning DAG models. Since
the prior enables an analytical evaluation of the marginal likelihood
of an LDAG, the model space can be searched relatively fast for structures
that are associated with high posterior probabilities. To do this
in practice, we combine a non-reversible Markov chain Monte Carlo
algorithm with a greedy hill climbing approach to obtain a method
that is not computationally too expensive. 

The structure of this article is as follows. In Section 2 we introduce
the LDAG models (\ref{subsec:Preliminaries}) and investigate their properties (\ref{sec:propLDAGs}). In Section 3 we
develop the Bayesian learning method and apply it to both real and
synthetic data sets in Section 4 to illustrate the favorable properties
of our approach. Some concluding remarks are provided in the final
section.

\section{DAG- and LDAG-based graphical models\label{sec:DAGs and LDAGs}}
\subsection{Preliminaries and introduction of LDAGs\label{subsec:Preliminaries}}
A DAG is a directed graph that is built up by nodes and directed edges. The acyclic property ensures that no directed path starting from a node leads back to that particular node. The concept of DAG-based graphical models, or Bayesian networks, was formalized by \citet{Pearl88ProbReason}. In a Bayesian network, the nodes represent variables and the directed edges represent direct dependencies among the variables. Correspondingly, absence of edges represents statements of conditional independence. The constraints imposed by the structure of a DAG alone have been recognized to be unnecessarily stringent under certain circumstances where context-specific or asymmetric independence can play a natural role in the models. In general, two approaches have been considered for this problem. The most common approach is based on different representations of conditional probability distributions (CPD) that are hidden behind the graph topology (\citet{Boutilier96CSIinBN,Chickering97abayesian,Poole03exploitingcontextual}). The other approach has focused on the topology of the graph structure itself (\citet{Heckerman91ProbSimNet,Geiger96multinets}). \citet{Heckerman91ProbSimNet} introduced similarity networks which are made up of multiple networks. This representation and the related Bayesian multinets are further examined in \citet{Geiger96multinets}. They show how asymmetric independencies can be represented by multiple Bayesian networks and how these independence assertions can speed up inference. 

In this paper we will bring the CPD- and graph-based approaches together by introducing a graphical representation scheme in the form of labeled DAGs whose associated CPDs can be stored in compact tables based on rules. To illustrate the concept of LDAGs, we consider the following example from \citet{Geiger96multinets}, p. 52:\\
\begin{adjustwidth}{0.75cm}{0cm}
A guard of a secured building expects three types of persons ($h$) to approach the building's entrance: workers in the buildings, approved visitors, and spies. As a person approaches the building, the guard can note its gender ($g$) and whether or not the person wears a badge ($b$). Spies are mostly men. Spies always wear badges in an attempt to fool the guard. Visitors don't wear badges because they don't have one. Female workers tend to wear badges more often than do male workers. The task of the guard is to identify the type of person approaching the building.\\
\end{adjustwidth}
This scenario can be represented by the DAG on top in Figure \ref{fig:spyEx}. The topology of this graph, however, hides the fact that gender and badge wearing are conditionally independent, given that the person is a spy or visitor. The corresponding joint probability distribution is, as a result of this, overparameterized in the sense that it requires a total of 11 free parameters although some of these are identical. \citet{Geiger96multinets} noticed that this scenario is better represented by the multiple graphs reproduced here in the middle of Figure \ref{fig:spyEx}. This representation is made up of two context-specific graphs that together show that the dependence between gender and badge wearing only holds in the context of the person being a worker. The corresponding joint probability distribution now only requires 9 free parameters. Now consider the labeled DAG on the bottom in Figure \ref{fig:spyEx}. We have added the label $\{spy,visitor\}$ to the edge $(g,b)$. This label implies that gender and badge are independent given that the person approaching the building is a spy or a visitor. Although an LDAG is global in its representation, it can still represent independencies that only hold in certain contexts. This allows it to represent the same dependence structure that requires multiple graphs using the multinet approach. As for the multinet approach, this type of representation requires 9 free parameters.
\begin{figure}
\begin{centering}
\includegraphics[width=\textwidth-5cm]{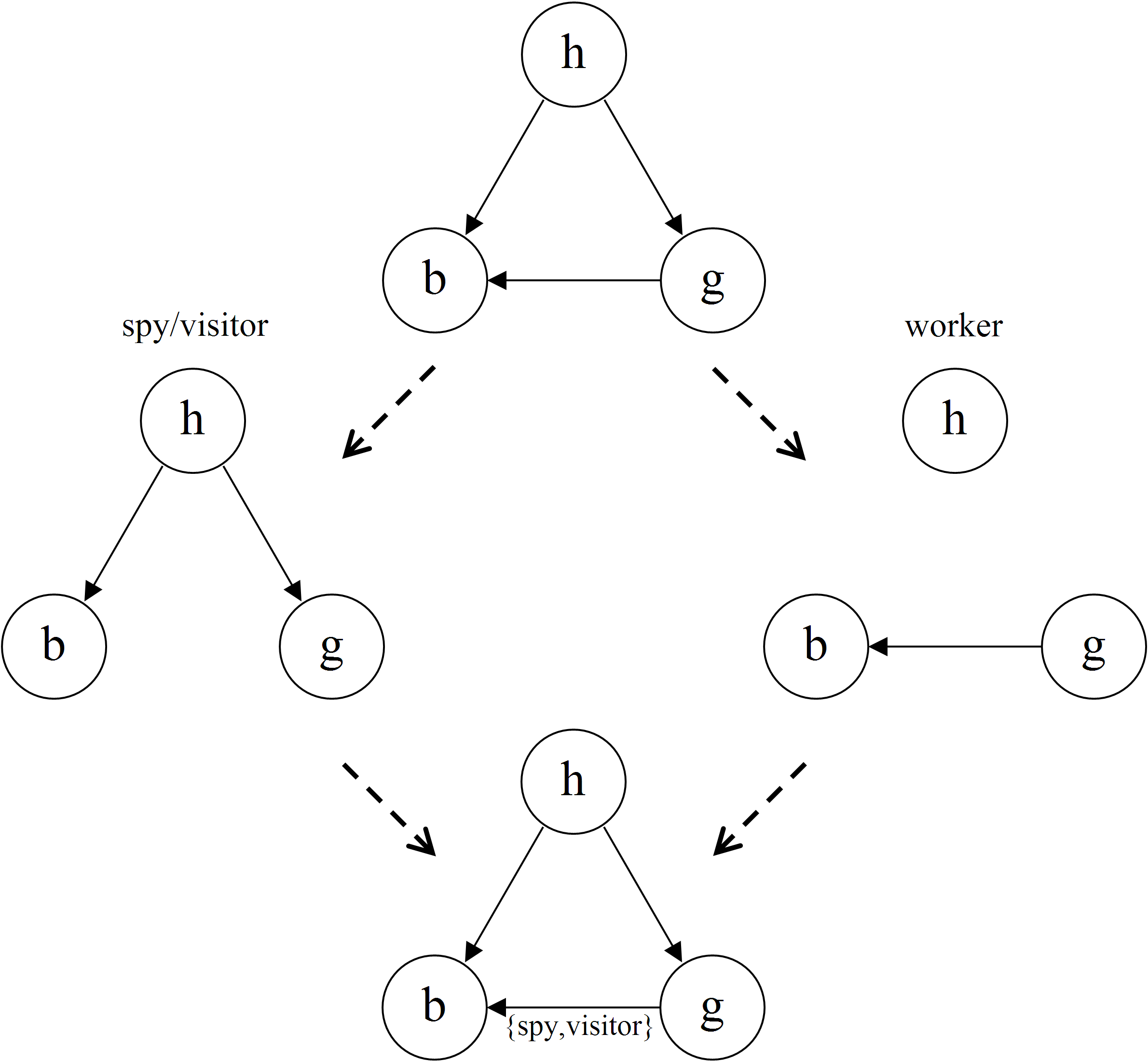}
\caption{Graph structures describing the spy/visitor/worker-scenario.\label{fig:spyEx}}
\end{centering}
\end{figure} 

Before stating any formal definitions, we provide some notations. A DAG will be denoted by $G=(V,E)$ where $V=\{1,\ldots,d\}$ is the set of nodes and $E\subset V\times V$ is the set of edges such that if $(i,j)\in E$ then the graph contains a directed edge from node $i$ to $j$. Nodes, from which there is a directed edge to node $j$, are called parents of $j$ and the set of all such nodes is denoted by $\Pi_{j}=\{ i\in V:(i,j)\in E\}$. The nodes $V$ give the indices of a set of stochastic variables $X=\{X_{1},\ldots,X_{d}\}$. Due to the close relationship between a node and its corresponding variable, the terms node and variable are used interchangeably. We use small letters $x_{j}$ to denote a value taken by the corresponding variable. If $S\subseteq V$, then $X_{S}$ denotes the corresponding set of variables. The outcome space of a variable $X_{j}$ is denoted by $\mathcal{X}_{j}$ and the joint outcome space of a set of variables by the Cartesian product $\mathcal{X}_{S}=\times_{j\in S}\mathcal{X}_{j}$. The cardinality of the outcome space of $X_{S}$ is denoted by $|\mathcal{X}_{S}|$. 

An ordinary DAG encodes independence statements in the form of conditional independencies. 
\begin{definition}\textit{Conditional Independence (CI)\label{Conditional-Independence}}\\
Let {$X=\{X_{1},\ldots,X_{d}\}$}
be a set of stochastic variables where $V=\{1,\ldots d\}$ and let $A$, $B$, $S$ be three disjoint subsets of $V$. $X_{A}$ is conditionally independent of $X_{B}$ given $X_{S}$ if  
\[
p(x_{A}\mid x_{B},x_{S})=p(x_{A}\mid x_{S})
\]
holds for all $(x_{A},x_{B},x_{S})\in\mathcal{X}_{A}\times\mathcal{X}_{B}\times\mathcal{X}_{S}$ whenever $p(x_{B},x_{S})>0$. This will be denoted by 
\[
X_{A}\perp X_{B}\mid X_{S}.
\]
\end{definition}
\noindent If we let $X_{S}=\varnothing$, then $X_{A}\perp X_{B}$
simply denotes marginal independence between the two sets of variables. The most basic statements of conditional independence, reflected by a DAG, follow the directed local Markov property. It implies that each variable $X_{j}$ is conditionally independent of its non-descendants given its parental variables $X_{\Pi_{j}}$. This leads to a unique explicit factorization of the joint distribution into lower order distributions,
\begin{equation}
p(X_{1},\ldots X_{d})=\overset{{\scriptstyle d}}{\underset{{\scriptstyle j=1}}{\prod}}p(X_{j}\mid X_{\Pi_{j}}),\label{eq:facDAG}
\end{equation}
where the factors are CPDs that correspond to local structures. By local structure, we refer to the node itself, its parents and the edges from the parents to the node.  The topology of an ordinary DAG restricts it to only encoding for independence relations that hold globally. However, as shown in the example above, it is natural to consider independence relations that only hold in certain contexts. The following notion of context-specific independence was formalized by \citet{Boutilier96CSIinBN}.   
\begin{definition}\textit{Context-Specific Independence (CSI)}\\ 
Let {$X=\{X_{1},\ldots,X_{d}\}$}
be a set of stochastic variables where $V=\{1,\ldots d\}$ and let $A$, $B$, $C$, $S$ be four disjoint subsets of $V$. $X_{A}$ is contextually independent of $X_{B}$ given $X_{S}$ and the context $X_{C}=x_{C}$ if  
\[
p(x_{A}\mid x_{B},x_{C},x_{S})=p(x_{A}\mid x_{C},x_{S})\,,
\]
holds for all $(x_{A},x_{B},x_{S})\in\mathcal{X}_{A}\times\mathcal{X}_{B}\times\mathcal{X}_{S}$ whenever $p(x_{B},x_{C},x_{S})>0$. This will be denoted by 
\[
X_{A}\perp X_{B}\mid x_{C},X_{S}.
\]
\end{definition}
It has been discovered by numerous authors that certain CSIs can naturally be captured simply by further refining (\ref{eq:facDAG}). We will refer to these statements as local CSIs as they are confined to the local structures.
\begin{definition}\textit{Local CSI in a DAG\label{Local CSI}}\\
A CSI in a DAG is local if it is of the form $X_{j}\perp X_{B}\mid x_{C}$, where $B$ and $C$ form a partition of the parents of node $j$. 
\end{definition}

In the CPD-based approaches to including CSI in Bayesian networks, the context-specific local structures cannot be read directly off the graph structure. This is the key to the usefulness of multinets. A multinet offers a natural representation of the dependence structure by explicitly showing the independencies in a graphical form. In an attempt to further pursue this idea, we introduce a graph topology that is able to visualize the local CSIs directly as a part of a single graph structure as done in Figure \ref{fig:spyEx}. To achieve this we add labels to the edges in a similar way as \citet{Corander03LGM}. This enables incorporation of local CSIs in a single graph as opposed to multiple networks -approaches where one might need one graph for each distinct context. An LDAG is now formally defined as a DAG with labels representing local CSIs. 
\begin{definition}\textit{Labeled Directed Acyclic Graph (LDAG)\label{Labelled-Directed-Acyclic}}\\
 Let $G=(V,E)$ be a DAG over the stochastic variables $\{X_{1},\ldots,X_{d}\}$. For all $(i,j)\in E$, let $L_{(i,j)}=\Pi_{j}\setminus\{i\}$. A label on an edge $(i,j)\in E$ is defined as the set 
\[
\mathcal{L}_{(i,j)}=\left\{ x_{L_{(i,j)}}\in\mathcal{X}_{L_{(i,j)}}:X_{j}\perp X_{i}\mid x_{L_{(i,j)}}\right\} .
\]
An LDAG is a DAG to which the label set $\mathcal{L}_{E}=\{\mathcal{L}_{(i,j)}:\mathcal{L}_{(i,j)}\neq\varnothing\}_{(i,j)\in E}$ has been added, it is denoted by $G_{L}=(V,E,\mathcal{L}_{E})$ 
\end{definition}
\begin{figure}
\begin{centering}
\includegraphics[width=5cm]{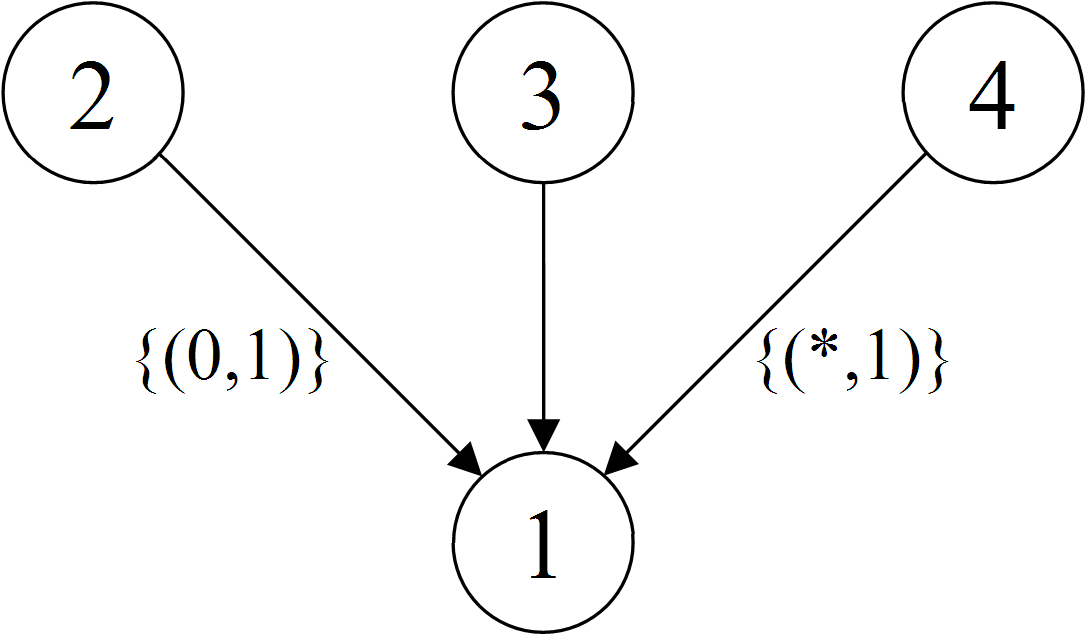}{\small 
\begin{eqnarray*}
\mathcal{L}_{(2,1)}=(0,1) &  & \Rightarrow X_{1}\perp X_{2}\mid(X_{3},X_{4})=(0,1)\\
\\
\mathcal{L}_{(4,1)}=\mathcal{X}_{2}\times \{1\} &  & \Rightarrow X_{1}\perp X_{4}\mid X_{2}\in \mathcal{X}_{2},X_{3}=1\\
 &  & \Leftrightarrow X_{1}\perp X_{4}\mid X_{2},X_{3}=1
\end{eqnarray*}}
\vspace{-0.5cm}
\caption{Local CSI-structure and the corresponding local CSIs.\label{fig:locCSIstructWstat}}
\par\end{centering}{\small \par}
\end{figure} 
With respect to a fixed ordering of the variables, the labels do not have to contain any variable indices as $X_{L_{(i,j)}}$ contains all the parental variables in the node's local structure except the one that is part of the edge. A node must naturally have at least two parents for it to be possible for an incoming edge to contain a label. In subsequent examples, we assume that the variables are binary with $\mathcal{X}_{j}=\{0,1\}$. However, the derived theory applies to non-binary variables as well. For $L_{(i,j)}=\{k,l\}$, we will use $(*,x_{l})$ to denote a label set of the form $\mathcal{X}_k \times \{x_{l}\}$. Figure \ref{fig:locCSIstructWstat} now illustrates how labels may be added to the edges of a local structure and how they should be interpreted. The number of configurations relative to the number of possible configurations in a label can be considered an indication of the strength of the dependence conveyed by the corresponding edge in that particular local structure. 

The CPD-based approach for generalizing Bayesian networks utilizes the fact that local CSIs correspond to certain regularities among the CPDs arising in the factorization (\ref{eq:facDAG}). The focus has therefore been to find different ways of representing the CPDs. \citet{Boutilier96CSIinBN} use decision trees in which certain local CSI-structures can be captured in a very natural way. However, due to the replication problem (\citet{Pagallo90RepProb}), trees are somewhat limited in their expression power when it comes to certain types of CSI-structure. \citet{Chickering97abayesian} overcome this shortcoming by using a more general type of representation in the form of decision graphs. Unfortunately, decision graphs usually leave the scope of CSI complicating the interpretability of the models from a (in)dependence perspective. In addition, such models lack the ability to exploit CSI in inference. Next we investigate how the tree- and decision graph-based approaches are connected to the LDAG representation which in fact can be considered a compromise between the two representations. LDAGs allow for more expressive models than CPT-trees without leaving the scope of CSI which provides a natural interpretation and proven computational advantages when performing inference.

The textbook way of representing the CPDs is in the form of full conditional probability tables (CPT) of sizes $(|\mathcal{X}_{j}|-1)\cdot|\mathcal{X}_{{\Pi}_{j}}|$. This form of representation requires tables that grow exponentially with the numbers of parents as it fails to capture regularities present in the CPDs. Including local CSIs in the model, however, directly implies that certain CPDs must be similar and need only be defined once. 
\begin{figure}
\begin{minipage}[t]{5cm}%
\begin{center}
\includegraphics[width=5cm]{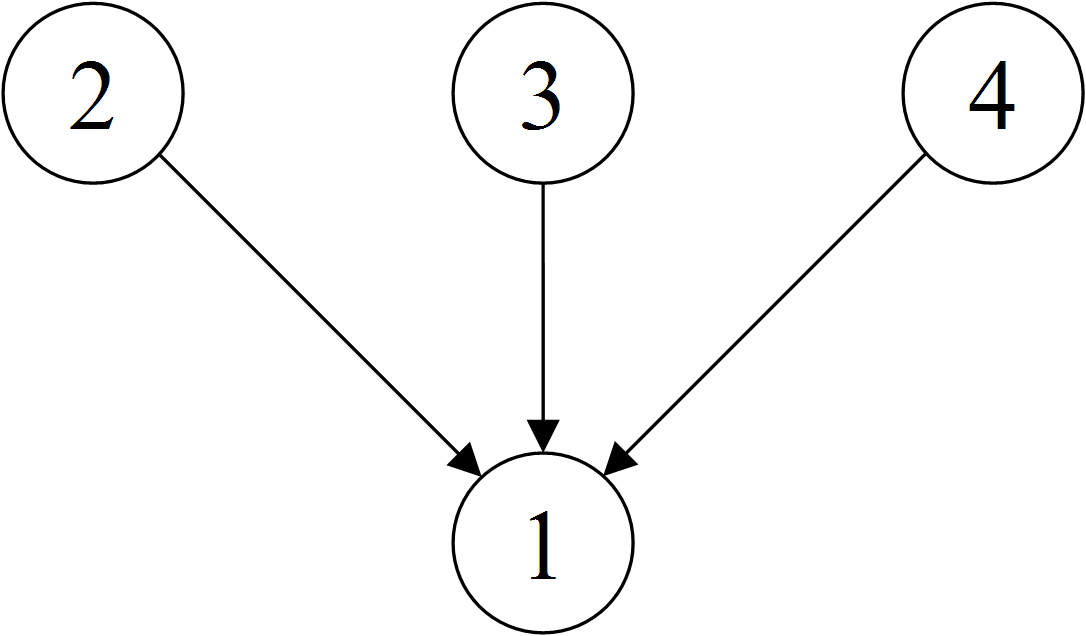}
\par\end{center}%
\end{minipage}%
\hspace{0.1cm}
\begin{minipage}[t]{\columnwidth-5cm}%
\begin{center}
\vspace{-3cm}
\begin{tabular}{|c|c|}
\hline\
{\small $X_{\Pi_{1}}$} & {\small $p(X_{1}\mid X_{\Pi_{1}})$}\tabularnewline
\hline 
{\scriptsize $X_{2}=0\wedge X_{3}=0\wedge X_{4}=0$} & {\scriptsize $p_{1}$}\tabularnewline
{\scriptsize $X_{2}=0\wedge X_{3}=0\wedge X_{4}=1$} & {\scriptsize $p_{3}$}\tabularnewline
{\scriptsize $X_{2}=0\wedge X_{3}=1\wedge X_{4}=0$} & {\scriptsize $p_{4}$}\tabularnewline
{\scriptsize $X_{2}=0\wedge X_{3}=1\wedge X_{4}=1$} & {\scriptsize $p_{4}$}\tabularnewline
{\scriptsize $X_{2}=1\wedge X_{3}=0\wedge X_{4}=0$} & {\scriptsize $p_{2}$}\tabularnewline
{\scriptsize $X_{2}=1\wedge X_{3}=0\wedge X_{4}=1$} & {\scriptsize $p_{3}$}\tabularnewline
{\scriptsize $X_{2}=1\wedge X_{3}=1\wedge X_{4}=0$} & {\scriptsize $p_{5}$}\tabularnewline
{\scriptsize $X_{2}=1\wedge X_{3}=1\wedge X_{4}=1$} & {\scriptsize $p_{5}$}\tabularnewline
\hline 
\end{tabular}
\par\end{center}%
\end{minipage}
\par
\caption{Local structure and the associated CPT.\label{fig:locStructWcpt}}
\end{figure}
Consider the local structure and associated CPT in Figure \ref{fig:locStructWcpt}. We use complete AND-rules to represent the distinct parent configurations. A rule is complete if all parental variables are part of it. By investigating the right column of the CPT, we see that there are only five distinct CPDs. Still, the naive approach requires us to define $p(X_{1}\mid x_{\Pi_{1}})$ for each distinct parent configuration $x_{\Pi_{1}}\in \mathcal{X}_{\Pi_{1}}$. It is therefore easy to realize that this representation is far from minimal.
\begin{figure}
\begin{minipage}[t]{5cm}%
\begin{center}
\includegraphics[width=5cm]{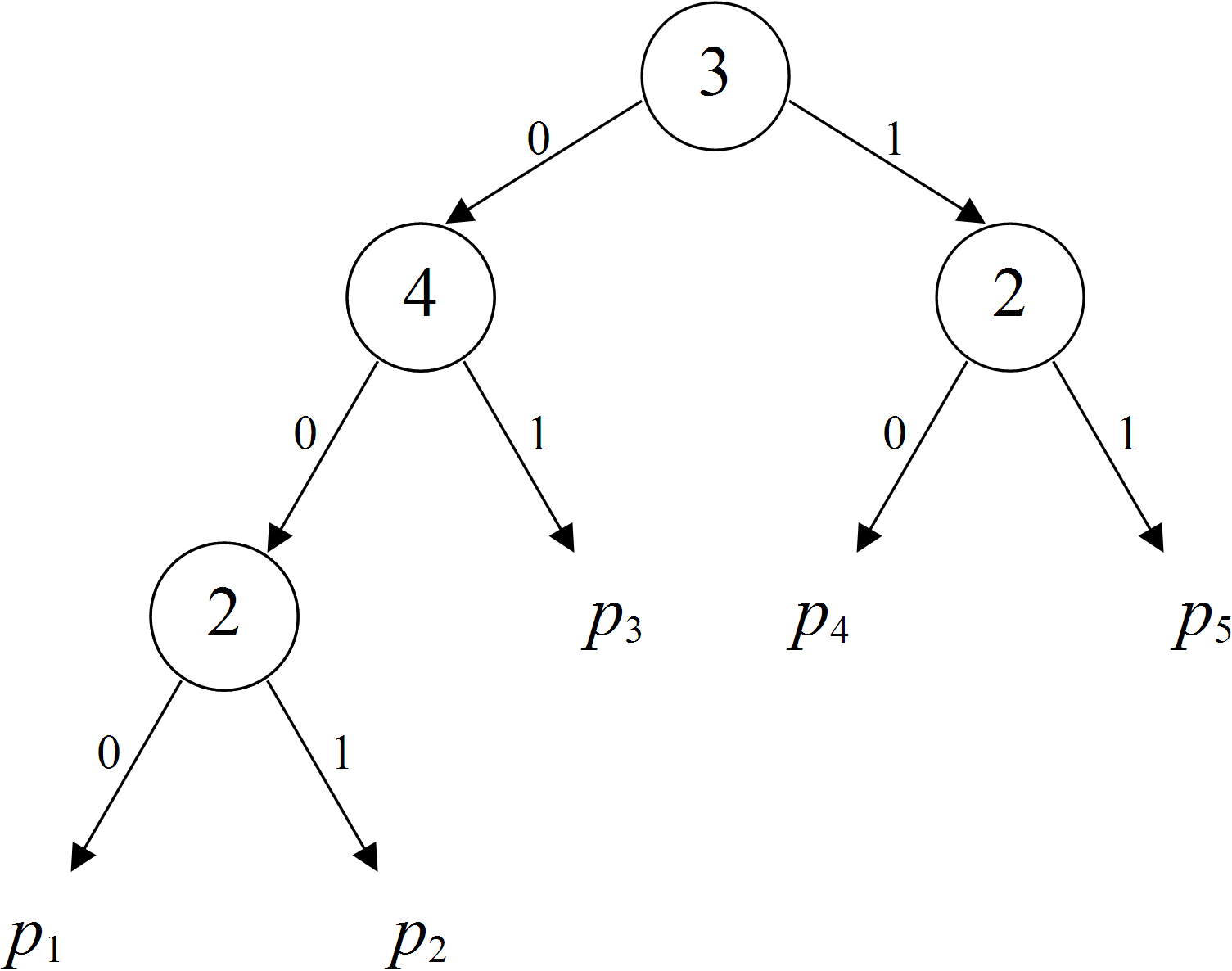}
\par\end{center}%
\end{minipage}%
\hspace{0.1cm}
\begin{minipage}[t]{\columnwidth-5cm}%
\begin{center}
\vspace{-3cm}
\begin{tabular}{|c|c|}
\hline 
{\small $X_{\Pi_{j}}$} & {\small $p(X_{1}\mid X_{\Pi_{j}})$}\tabularnewline
\hline
{\scriptsize $X_{3}=0\wedge X_{4}=0\wedge X_{2}=0$} & {\scriptsize $p_{1}$}\tabularnewline
{\scriptsize $X_{3}=0\wedge X_{4}=0\wedge X_{2}=1$} & {\scriptsize $p_{2}$}\tabularnewline
{\scriptsize $X_{3}=0\wedge X_{4}=1$} & {\scriptsize $p_{3}$}\tabularnewline 
{\scriptsize $X_{3}=1\wedge X_{2}=0$} & {\scriptsize $p_{4}$}\tabularnewline
{\scriptsize $X_{3}=1\wedge X_{2}=1$} & {\scriptsize $p_{5}$}\tabularnewline
\hline 
\end{tabular}
\par\end{center}%
\end{minipage}
\par
\caption{CPT-tree and the corresponding rule-based reduced CPT.\label{fig:CPTtreeWredCPT}}
\end{figure}
Using the approach of \citet{Boutilier96CSIinBN}, the regularities in the CPT in Figure \ref{fig:locStructWcpt} can be captured by the CPT-tree in Figure \ref{fig:CPTtreeWredCPT}. Each path in the tree corresponds to a rule that can be described by the AND-operator. By simply traversing down each distinct path until we reach a terminal node or leaf, we can transform the CPT in Figure \ref{fig:locStructWcpt} into its reduced counterpart on the right in Figure \ref{fig:CPTtreeWredCPT}. All parent configurations satisfying a certain rule give rise to the same CPD. This implies that the rules in a reduced CPT must be mutually exclusive for the representation to be minimal. The rules corresponding to a tree are mutually exclusive as two distinct paths cannot lead to the same leaf. If a variable is not part of a path (or the corresponding AND-rule), it implies that the particular variable is contextually independent of the variable associated with the CPT given the context encoded by the path (or rule). Following this method we can read off the following local CSIs: 
\[
\begin{array}{c}
\begin{array}{ccc}
X_{3}=0\wedge X_{4}=1 & \Rightarrow & X_{1}\perp X_{2}\mid (X_{3},X_{4})=(0,1)\end{array}\hspace*{4.125cm}\\
\left.\begin{array}{ccc}
X_{3}=1\wedge X_{2}=0 & \Rightarrow & X_{1}\perp X_{4}\mid (X_{2},X_{3})=(0,1)\\
X_{3}=1\wedge X_{2}=1 & \Rightarrow & X_{1}\perp X_{4}\mid (X_{2},X_{3})=(1,1)
\end{array}\right\} \Leftrightarrow X_{1}\perp X_{4}\mid X_{2},X_{3}=1
\end{array}
\]
If we once more consider Figure \ref{fig:locCSIstructWstat}, we see that the CSIs above coincide with the labels of this specific LDAG. More generally, any CPT-tree can be transformed into a reduced CPT by mutually exclusive AND-rules. Subsequently, incomplete rules can be turned into labels as illustrated in the above example.

Now consider the LDAG on the top in Figure \ref{fig:locCSIstructWtreeWdc} and its associated minimal reduced CPT on the bottom in Figure \ref{fig:CPTtoREDCPT}. This CSI-structure cannot be compactly represented by the structure of a CPT-tree.
\begin{figure}
\begin{center}
\includegraphics[width=\columnwidth-2cm]{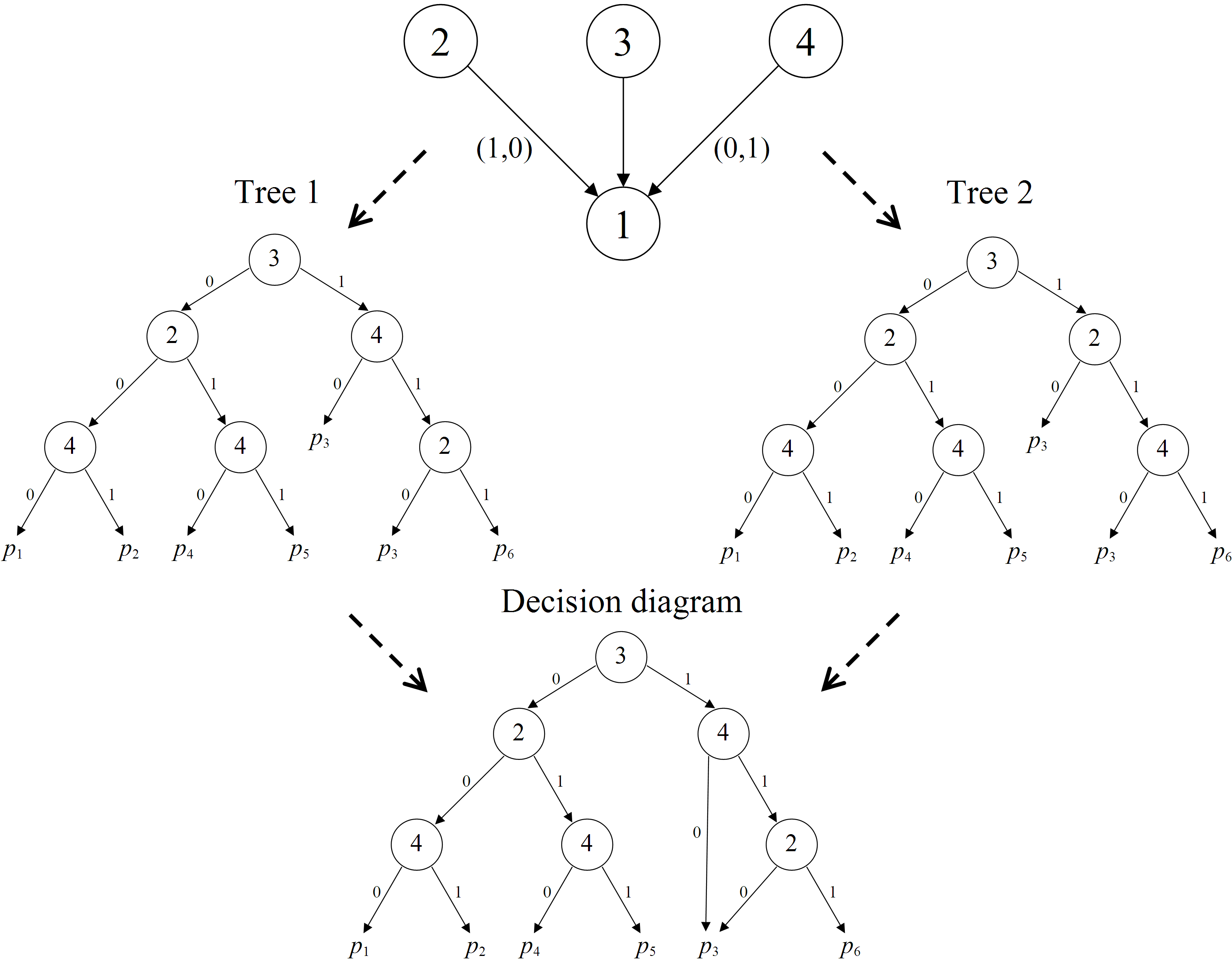}
\caption{Tree- and graph-based representations of a local CSI-structure.\label{fig:locCSIstructWtreeWdc}}
\end{center}
\end{figure} 
Depending on the order of the variables, we can at best represent either of the CSIs (tree 1 or tree 2 in Figure \ref{fig:locCSIstructWtreeWdc}). As a result, duplicate subtrees can be found in both trees. This illustrates the weakness of tree structures. Once we split on a variable, it is rendered essential for the particular context even if it in the end does not affect the CPD. To represent this type of CSI-structures through a graphical representation, we need the more general decision graph used by \citet{Chickering97abayesian}. As a node in a decision graph may have multiple parents, we can merge leaves with similar CPDs in a tree. Merging such leaves in tree 1 in Figure \ref{fig:locCSIstructWtreeWdc} results in the decision graph at the bottom in which all CSIs are represented. When merging leaves in a CPT-tree, situations may arise where some of the corresponding AND-rules are not mutually exclusive anymore. In order to achieve a minimal representation of the reduced CPT, any mutually non-exclusive AND-rules must be combined with the OR-operator.

To show how the minimal reduced CPT is recovered from the labels we proceed stepwise as shown in Figure \ref{fig:CPTtoREDCPT}. Each of the labels corresponds to a single reduced AND-rule resulting in the upper table. The rules on row 3 and 6 are not mutually exclusive at this point as they both are satisfied by the common parent configuration $(X_2,X_3,X_4)=(0,1,0)$. This implies that any parent configuration satisfying any of these rules must give rise to the same CPD. The AND-rules are therefore combined with the OR-operator resulting in the minimal reduced CPT on the bottom of the figure. More generally, each configuration in the labels of a local structure corresponds to an AND-rule. If any two rules overlap, they are combined with the OR-operator. The rules of a minimal reduced CPT created by this method will thus be mutually exclusive and exhaustive with respect to the outcome space of the parental variables. 
\begin{figure}
\begin{centering}
\begin{tabular}{|c|c|}
\hline 
{\small $X_{\Pi_{1}}$} & {\small $p(X_{1}\mid X_{\Pi_{1}})$}\tabularnewline
\hline 
{\scriptsize $X_{2}=0\wedge X_{3}=0\wedge X_{4}=0$} & {\scriptsize $p_{1}$}\tabularnewline
{\scriptsize $X_{2}=0\wedge X_{3}=0\wedge X_{4}=1$} & {\scriptsize $p_{2}$}\tabularnewline
{\scriptsize $X_{2}=0\wedge X_{3}=1$} & {\scriptsize $p_{3}$}\tabularnewline
{\scriptsize $X_{2}=1\wedge X_{3}=0\wedge X_{4}=0$} & {\scriptsize $p_{4}$}\tabularnewline
{\scriptsize $X_{2}=1\wedge X_{3}=0\wedge X_{4}=1$} & {\scriptsize $p_{5}$}\tabularnewline
{\scriptsize $X_{3}=1\wedge X_{4}=0$} & {\scriptsize $p_{3}$}\tabularnewline
{\scriptsize $X_{2}=1\wedge X_{3}=1\wedge X_{4}=1$} & {\scriptsize $p_{6}$}\tabularnewline
\hline 
\end{tabular}

\vspace{0.25cm}
\huge $\Downarrow \ \ \Downarrow \ \ \Downarrow$ \small
\vspace{0.25cm}

\begin{tabular}{|c|c|}
\hline 
{\small $X_{\Pi_{1}}$} & {\small $p(X_{1}\mid X_{\Pi_{1}})$}\tabularnewline
\hline 
{\scriptsize $X_{2}=0\wedge X_{3}=0\wedge X_{4}=0$} & {\scriptsize $p_{1}$}\tabularnewline
{\scriptsize $X_{2}=0\wedge X_{3}=0\wedge X_{4}=1$} & {\scriptsize $p_{2}$}\tabularnewline
{\scriptsize $(X_{2}=0\wedge X_{3}=1) \vee (X_{3}=1\wedge X_{4}=0)$} & {\scriptsize $p_{3}$}\tabularnewline
{\scriptsize $X_{2}=1\wedge X_{3}=0\wedge X_{4}=0$} & {\scriptsize $p_{4}$}\tabularnewline
{\scriptsize $X_{2}=1\wedge X_{3}=0\wedge X_{4}=1$} & {\scriptsize $p_{5}$}\tabularnewline
{\scriptsize $X_{2}=1\wedge X_{3}=1\wedge X_{4}=1$} & {\scriptsize $p_{6}$}\tabularnewline
\hline 
\end{tabular}
\caption{Constructing a minimal reduced CPT through a multi-step process.\label{fig:CPTtoREDCPT}}
\end{centering}
\end{figure} 

A CPT-representation may in fact be viewed as a function that given a parent configuration returns a CPD. The common factor among the different CPD-based representations is that they all induce partitions of the outcome space of the parental variables. If the representation is based on the notion of CSI, the corresponding partition will be referred to as CSI-consistent. Decision graphs in general go beyond the scope of CSI as they are able to represent any arbitrary partition. Subsequently, a decision graph must fulfill certain criteria structure-wise for it to be consistent with CSI. Still, even if a decision graph indeed is consistent with respect to CSI, the interpretation of the local CSIs is not trivial. In an LDAG, however, all local CSIs can readily be recovered from the labels. By introducing the class of LDAGs we aim at balancing the expressive power of the models against their interpretability. CSI has a sound interpretation and arises naturally in various situations. From a computational perspective, CSI has also proven to be particularly useful in probabilistic inference.  

Considerable research efforts have been devoted to outlining how CSI can be exploited in probabilistic inference. Probabilistic inference refers to the process of computing the posterior probability of a list of query variables given some observed variables. The key to efficient inference lies in the concept of factorization of the joint distribution. Incorporating local CSIs into the models, allows a further decomposition of (\ref{eq:facDAG}) into a finer-grained factorization which in turn can speed up the inference. \citet{Boutilier96CSIinBN} investigate how CPT-trees can be used to speed up various inference algorithms. As a consequence of the replication problem, \citet{Poole97ruleInference} concludes that rule-based versions may be more efficient than tree-based. \citet{Zhang99roleCSIinference} give a more general analysis of the computational leverages that CSI has to offer without referring to any particular inference algorithms. \citet{Poole03exploitingcontextual} further improve the efficiency of the approach of \citet{Poole97ruleInference} by using the concept of confactors which is a combination of contexts and tables. They introduce contextual belief networks which are similar to traditional Bayesian networks except that the CPDs are associated with parent contexts rather than explicit parent configurations. The labels in an LDAG correspond directly to the parent contexts of a contextual belief network. However, in the process of making the contexts mutually exclusive they proceed in a different manner than in Figure \ref{fig:CPTtoREDCPT}. Their approach is more beneficial inference-wise but less compact, which interferes with the learning procedure we discuss in the next section. It is worth noting, however, that one may simply choose the approach more suitable for the problem at hand.    

Inference that exploits CSI has been quite thoroughly investigated by numerous authors. In this paper we thus focus on model identifiability and learning. From a practical point of view, one might argue that the sole existence of expressive and efficient models is not enough if these models cannot be accurately learned from a set of data. The more expressive the models, the harder the learning tends to get due to added flexibility. In section \ref{sec:Bayesian-learning-of-LDAGs} we present a Bayesian learning scheme for LDAGs but first we attend the aspect of model identifiability and interpretability. 

\subsection{Properties of LDAGs\label{sec:propLDAGs}}
To facilitate the interpretation of the CSI-structure of LDAGs, we utilize two conditions introduced for labeled undirected graphical models (LGMs) in \citet{Corander03LGM}; maximality and regularity.

Given a local structure, different  label combinations may induce the same local CSI-structure. This phenomenon may arise when the contexts of two or more labels overlap. If two distinct label combinations induce equivalent CSI-structures, it implies that the dissimilar label configurations can be added to each of the label combinations without inducing any new restrictions. To avoid this type of ambiguities we introduce the maximality condition for LDAGs.
\begin{definition}\textit{Maximal LDAG\label{MaximalDEF}}\\
An LDAG $G_{L}=(V,E,\mathcal{L}_{E})$ is called maximal if there exists no configuration $x_{L_{(i,j)}}$ that can be added to the label $\mathcal{L}_{(i,j)}$ without inducing an additional local CSI.
\end{definition}
If we add a configuration to a label in a maximal LDAG, it must result in an additional restriction in form of a local CSI. This will in turn result in a reduction of the associated minimal reduced CPT.
\begin{theorem}
Let $G_{L}=(V,E,\mathcal{L}_{E})$ and $G_{L}^{*}=(V,E,\mathcal{L}_{E}^{*})$
be two maximal LDAGs with the same underlying DAG $G=(V,E)$. Then $G_{L}$ and $G_{L}^{*}$ represent equivalent dependence structures if and only if $\mathcal{L}_{E}=\mathcal{L}_{E}^{*}$, i.e. $G_{L}=G_{L}^{*}$.  
\end{theorem}
\begin{proof}
Assume that $G_{L}=(V,E,\mathcal{L}_{E})$ and $G_{L}^{*}=(V,E,\mathcal{L}_{E}^{*})$ are two maximal LDAGs representing the same dependence structure and having the same underlying DAG $G=(V,E)$. Assume further that they have different labelings, i.e. $\mathcal{L}_{E}\not=\mathcal{L}_{E}^{*}$. There must thereby exist a configuration, say, $x_{L_{(i,j)}}^{*}\in \mathcal{L}_{(i,j)}^{*}$ such that $x_{L_{(i,j)}}^{*}\not\in \mathcal{L}_{(i,j)}$. This corresponds to the local CSI $X_{j} \perp X_{i} \mid x_{L_{(i,j)}}^{*}$ being explicitly represented in $G_{L}^{*}$ but not in $G_{L}$. For the LDAGs to represent the same dependence structure, the local CSI must, however, hold for $G_{L}$ as well. If the local CSI is implicitly represented by other labels in $G_{L}$, it implies that $x_{L_{(i,j)}}^{*}$ can be added to $\mathcal{L}_{(i,j)}$ with inducing an additional local CSI. This contradicts the maximality condition. If the CSI is not implicitly represented by other labels in $G_{L}$, it implies that adding $x_{L_{(i,j)}}^{*}$ to $\mathcal{L}_{E}$ induces an additional local CSI. This contradicts the assumption of the LDAGs representing the same dependence structure. This leads us to the conclusion that  $G_{L}$ and  $G_{L}^{*}$ represent the same dependence structures if and only if $\mathcal{L}_{E}=\mathcal{L}_{E}^{*}$. $\square$
\end{proof}
Each configuration in a label corresponds to a local CSI according to Definition \ref{Labelled-Directed-Acyclic}. If an LDAG is not maximal, some CSIs cannot be obtained directly from the graph following the definition. These CSIs are, however, implicitly reflected by the graph as they arise from a combination of other CSIs explicitly represented by labels in the graph. In maximal LDAGs, all local CSIs can be obtained directly from the graph. 
\begin{figure}
\hspace{0.3cm}%
\begin{minipage}[t]{0.30\columnwidth}%
\begin{center}
\includegraphics[width=5.2cm]{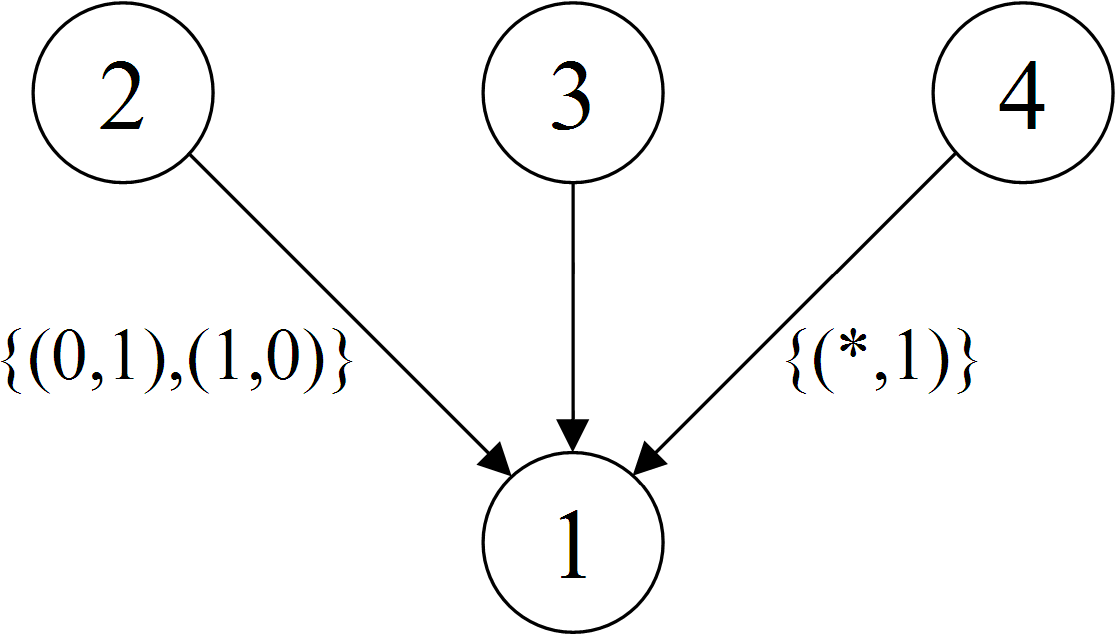}
\par\end{center}%
\end{minipage}\hspace{1.5cm}%
\begin{minipage}[t]{0.18\columnwidth}%
\begin{center}
\vspace{0.35cm}
\begin{tabular}{|c|c|}
\hline 
{\small $X_{\Pi_{1}}\in$} & {\small $p(X_{1}|X_{\Pi_{1}})$}\tabularnewline
\hline 
{\small $X_{2}=0\wedge X_{3}=0\wedge X_{4}=0$} & {\small $p_{1}$}\tabularnewline
{\small $X_{3}=0\wedge X_{4}=1$} & {\small $p_{3}$}\tabularnewline
{\small $ X_{3}=1 $} & {\small $p_{4}$}\tabularnewline
{\small $X_{2}=1\wedge X_{3}=0\wedge X_{4}=0$} & {\small $p_{2}$}\tabularnewline
\hline 
\end{tabular}
\par\end{center}%
\end{minipage}
\caption{Local CSI-structure and the associated minimal reduced CPT.\label{fig:LocalstructMaxEx}}
\end{figure}

To illustrate the maximality property, consider the local structure in Figure \ref{fig:LocalstructMaxEx}. The local structure is similar to the one in Figure \ref{fig:locCSIstructWstat} except that configuration $x_{L_{(2,1)}}=(1,0)$ has been added to its label. The local structure in Figure \ref{fig:LocalstructMaxEx} is now not maximal since $(1,1)$ can be added to $\mathcal{L}_{(2,1)}$ without resulting in an additional CSI. An intuitive way of reaching this conclusion is to consider the rules in the associated minimal reduced CPT next to the local structure. The rule on the third row arises from a combination of mutually non-exclusive rules:
\[
\left.\begin{array}{ccc}
x_{L_{(2,1)}}=(1,0) & \Rightarrow & X_{3}=1\wedge X_{4}=0\\
x_{L_{(4,1)}}=(0,1) & \Rightarrow & X_{2}=0\wedge X_{3}=1\\
x_{L_{(4,1)}}=(1,1) & \Rightarrow & X_{2}=1\wedge X_{3}=1
\end{array}\right\} \Rightarrow X_{3}=1
\]

\noindent As $(0,1,1)$ and $(1,1,1)$ satisfy this rule, no further merging of rules is done when $x_{L_{(2,1)}}=(1,1)$ is added to its label. This corresponds to the local CSI \[ X_{1}\perp X_{2}\mid X_{3}=1,X_{4}=1 \] implicitly being encoded by the other labels. This type of situation may arise when different label induced rules overlap and are combined with the OR-operator in order to achieve a minimal number of mutually exclusive rules.

The maximality condition is proven to be an essential condition for LDAGs. Without it we may fail in the interpretation of both local and, consequently, non-local CSIs which we consider later in this section. Failing in the interpretation of local CSIs hampers the efficiency of inference algorithms as useful CSIs may be neglected. A naive approach for testing whether an LDAG is maximal or not is simply to test each configuration that is not part of a label by adding it and checking if it results in a combining of rules or not. If there exists a configuration $x_{L_{(i,j)}}\not\in \mathcal{L}_{(i,j)}$ for which all parent contexts $\{x_{L_{(i,j)}}\}\times \mathcal{X}_{i}$ satisfy the same rule in the associated minimal reduced CPT, the LDAG is not maximal.

To ensure that the effect of an edge cannot completely vanish due to labels, we introduce the regularity condition for maximal LDAGs. 
\begin{definition}\textit{Regular maximal LDAG\label{RegularDEF}}\\
A maximal LDAG $G_{L}=(V,E,\mathcal{L}_{E})$ is regular
if $\mathcal{L}_{(i,j)}$ is a strict subset of $\mathcal{X}_{L_{(i,j)}}$
for every label in $G_{L}$. 
\end{definition}
\begin{figure}
\begin{centering}
\subfloat[Regular but not maximal.\label{fig:An-LDAG-satisfying}]{\includegraphics[height=2.88cm]{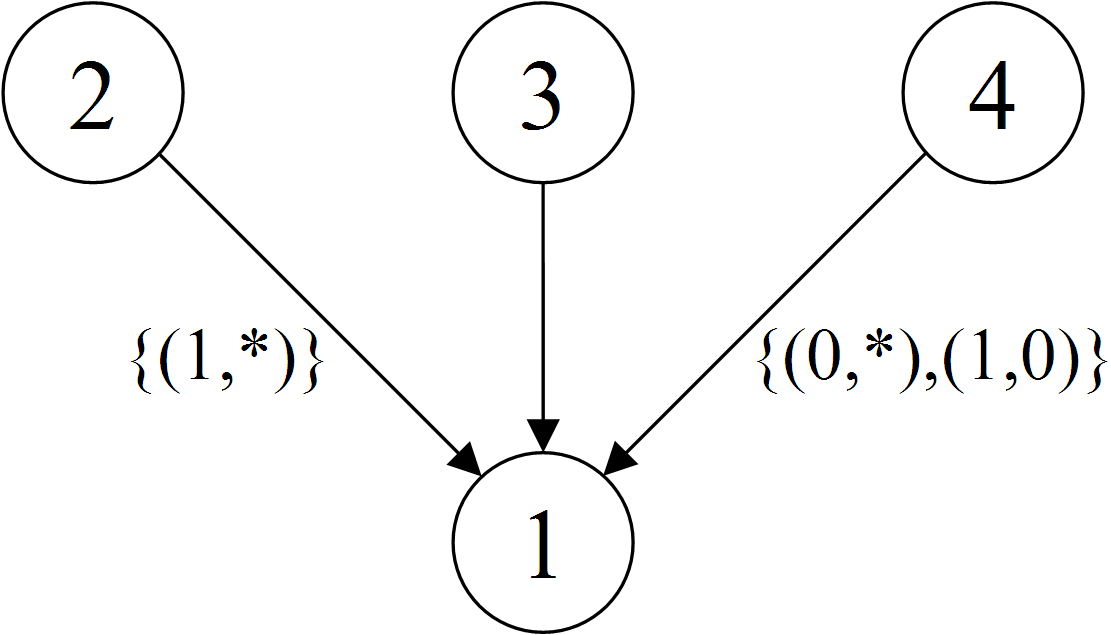}

}~~~~~\subfloat[Maximal but not regular.\label{fig:The-maximal-version} ]{\includegraphics[height=2.88cm]{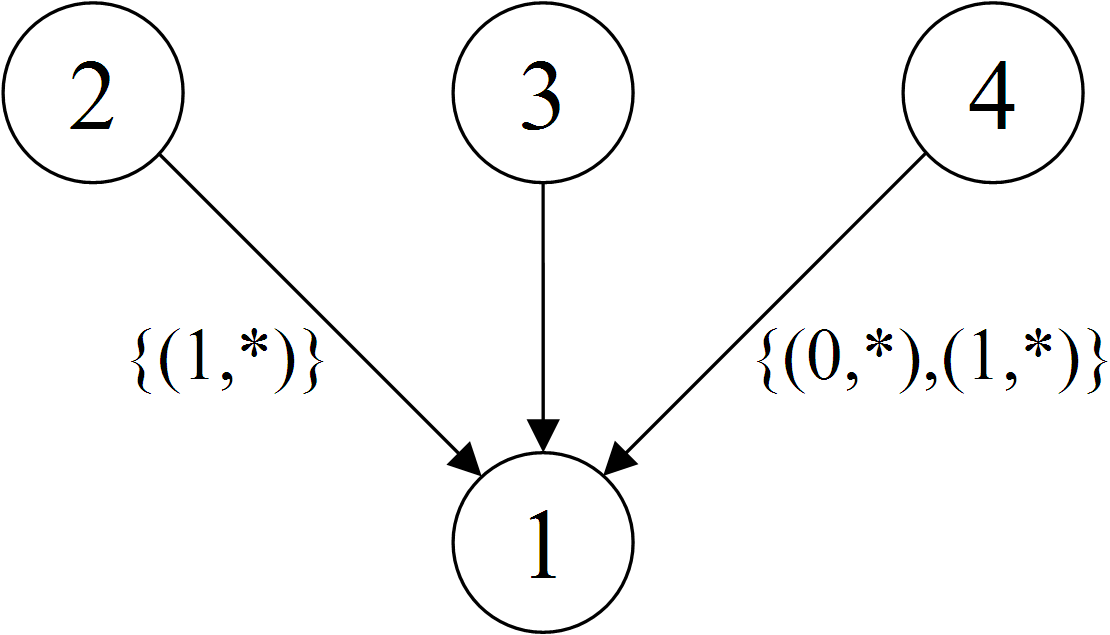}

}
\caption{Regularity condition for a maximal LDAG.\label{fig:Regularity-condition-for}}
\par\end{centering}
\end{figure}
The regularity condition is illustrated in Figure \ref{fig:Regularity-condition-for}. The LDAG in Figure \ref{fig:An-LDAG-satisfying} satisfies the regularity condition, however, this LDAG is not maximal. To make the LDAG maximal, the configuration $(1,1)$ must be added to $\mathcal{L}_{(4,1)}$ (Figure \ref{fig:The-maximal-version}) which now contains all possible configurations and thereby renders the maximal LDAG non-regular. 
\begin{theorem}\label{regmax edgerem}
In a regular maximal LDAG, a label $\mathcal{L}_{(i,j)}$ cannot induce an independence assertion of the form $X_{j} \perp X_{i} \mid x_{L_{(i,j)}}$ for all $x_{L_{(i,j)}} \in \mathcal{X}_{L_{(i,j)}}$, i.e. $X_{j} \perp X_{i} \mid X_{L_{(i,j)}}$.
\end{theorem}
\begin{proof}
Assume that $\mathcal{L}_{(i,j)}$ is a label in a regular maximal LDAG $G_{L}$. The maximality condition ensures that we cannot add configurations to $\mathcal{L}_{(i,j)}$ without altering the dependence structure. This means that $X_{j} \perp X_{i} \mid x_{L_{(i,j)}}$ must hold for $x_{L_{(i,j)}}\in\mathcal{L}_{(i,j)}$ but not for $x_{L_{(i,j)}} \in \mathcal{X}_{L_{(i,j)}}\setminus \mathcal{L}_{(i,j)}$. Due to the regularity condition, we have that $\mathcal{L}_{(i,j)} \subset \mathcal{X}_{L_{(i,j)}}$ and, consequently, $\mathcal{X}_{L_{(i,j)}}\setminus \mathcal{L}_{(i,j)}\not=\varnothing$. It must thereby exist an outcome $x_{L_{(i,j)}}^{*}$ for which $X_{j} \not\perp X_{i} \mid x_{L_{(i,j)}}^{*}$. $\square$
\end{proof}
\begin{figure}
\begin{centering}
\includegraphics[height=2.88cm]{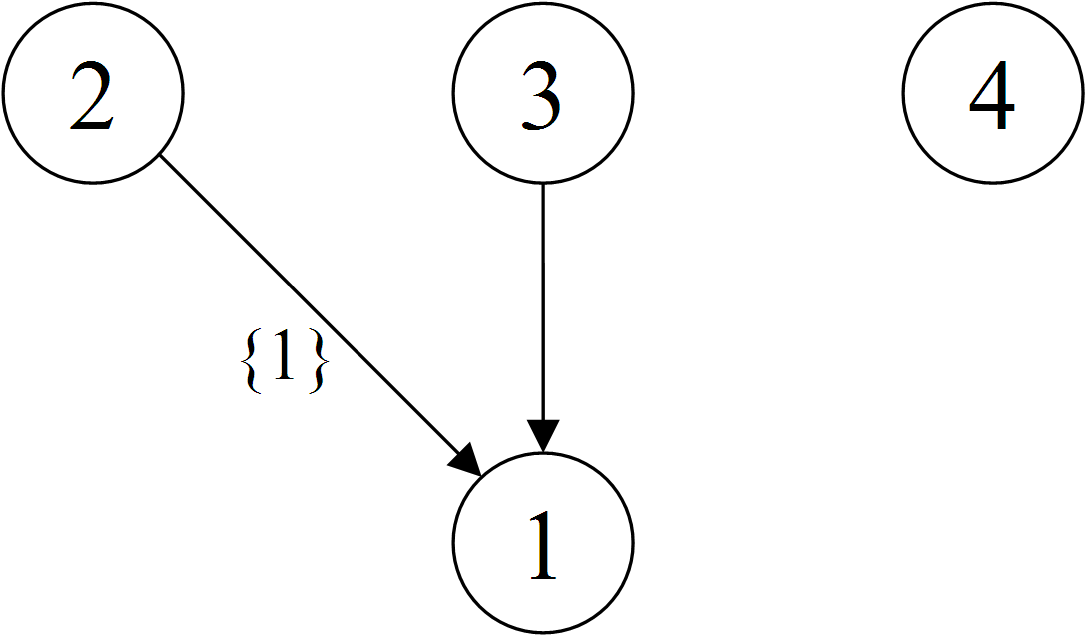}
\caption{Simplified version of the LDAG from Figure \ref{fig:The-maximal-version}.\label{fig:Simplified-version-of}}
\par\end{centering}
\end{figure}
If we consider the non-regular LDAG in Figure \ref{fig:The-maximal-version}, it can be simplified as the maximal regular LDAG in Figure \ref{fig:Simplified-version-of}
without altering the dependence structure. We can restrict our model space to the class of maximal regular LDAGs which is considerably smaller than the class of all LDAGs, without  suffering any loss of generality as the dependence structure of an arbitrary LDAG can be represented by a maximal regular LDAG.  

The independence assertions that can be recovered directly or indirectly from the structure of an LDAG $G_{L}$ can be divided into local and non-local. The local CIs  follow from the directed local Markov property while the local CSIs can be attained from the labels. The set of all local independencies associated with $G_{L}$ will be denoted by $\mathcal{I}_{loc}(G_{L})$. In addition to the local independencies, there are additional non-local independencies which can be derived from $\mathcal{I}_{loc}(G_{L})$. The set of all local and non-local independencies, denoted by $\mathcal{I}(G_{L})$, fully describes the dependence structure of $G_{L}$. However, the dependence structure of an LDAG or $\mathcal{I}{(G_{L})}$ is still ultimately determined by the local independencies  since all non-local independencies are implictly represented by $\mathcal{I}_{loc}(G_{L})$.

Let $P$ denote a distribution over the same set of variables as an LDAG $G_{L}$ and let $\mathcal{I}(P)$ denote the set of CSIs satisfied by $P$. If $P$ factorizes according to $G_{L}$, it must hold that $\mathcal{I}_{loc}(G_{L})\subseteq \mathcal{I}(G_{L}) \subseteq \mathcal{I}(P)$ and $G_{L}$ is called a CSI-map of $P$. There may, however, exist distribution-specific independencies that hold in $P$ even when they are not represented by the structure of $G_{L}$. A distribution $P$ is said to be faithful to $G_{L}$ if equality $\mathcal{I}(G_{L}) = \mathcal{I}(P)$ holds. The LDAG is then called a perfect CSI-map of $P$ and can be considered a true representation in the sense that no artificial dependencies are introduced.
	
The derivation of non-local CIs in ordinary DAGs can be very cumbersome. Instead, non-local CIs can be verified utilizing the concept of \emph{d}-separation. \citet{Boutilier96CSIinBN} introduce a sound counterpart for context-specific independence; CSI-separation. They reduce the problem of checking for CSI-separation by checking for ordinary variable independence in a simpler context-specific graph. To formulate the concept of CSI-separation for LDAGs, the following notions are introduced.
\begin{definition}\textit{Satisfied label}\\
Let $G_{L}=(V,E,\mathcal{L}_{E})$ be an LDAG and $X_{C}=x_{C}$ a context where $C\subseteq V$. In the context $X_{C}=x_{C}$, a label  $\mathcal{L}_{(i,j)}\in\mathcal{L}_{E}$ is satisfied if $L_{(i,j)}\cap C\not= \varnothing$ and $\{x_{L_{(i,j)}\cap C}\}\times \mathcal{X}_{{L}_{(i,j)}\setminus C}\subseteq \mathcal{L}_{(i,j)}$. 
\end{definition}
\begin{definition}\textit{Context-specific LDAG}\\
Let $G_{L}=(V,E,\mathcal{L}_{E})$ be an LDAG. For the context $X_{C}=x_{C}$, where $C\subseteq V$, the context-specific LDAG is denoted by $G_{L}(x_{C})=(V,E\setminus E',\mathcal{L}_{E\setminus E'})$ where $E'= \{ (i,j)\in E:\mathcal{L}_{(i,j)}\ is\ satisfied \}$. The underlying DAG of the context-specific LDAG is denoted by $G(x_{C})=(V,E\setminus E')$.
\end{definition}
A context-specific LDAG is a reduced version of an LDAG where all satisfied edges are removed. CSI-separation can now be defined in a similar manner as in \citet{Boutilier96CSIinBN}.
\begin{definition}\textit{CSI-separation in LDAGs}\\
Let $G_{L}=(V,E,\mathcal{L}_{E})$ be an LDAG and let $A$,$B$,$S$,$C$ be four disjoint subsets of $V$. $X_{A}$ is CSI-separated from $X_{B}$ by $X_{S}$ in the context $X_{C}=x_{C}$ in $G_{L}$, denoted by
\[
X_{A}\perp X_{B}\mid\mid_{G_{L}}x_{C},X_{S},
\] 
if $X_{A}$ is \emph{d}-separated from $X_{B}$ by $X_{S\cup C}$ in $G(x_{C})$.
\end{definition}
If $C=\varnothing$ in the above definition, the method describes the procedure of \emph{d}-separation with respect to the underlying DAG. CSI-separation is proven to be a sound method for verifying CSIs, i.e. 
\[
X_{A}\perp X_{B} \mid\mid_{G_{L}} x_{C},X_{S} \Rightarrow X_{A}\perp X_{B} \mid x_{C},X_{S}.
\]
Unfortunately, it is not complete in the sense that there may arise situations where certain structure induced independencies cannot be discovered directly by the CSI-separation algorithm. \citet{Koller+Friedman:09} noticed that it may be necessary to perform reasoning by cases to recover all independencies reflected by CSI-based structures. If CSI-separation holds for every $x_{C} \in \mathcal{X}_{C}$ in the above definition it will imply conditional independence. 
\begin{figure}
\begin{centering}
\subfloat[\label{fig:specialCase1}]{\includegraphics[height=2.88cm]{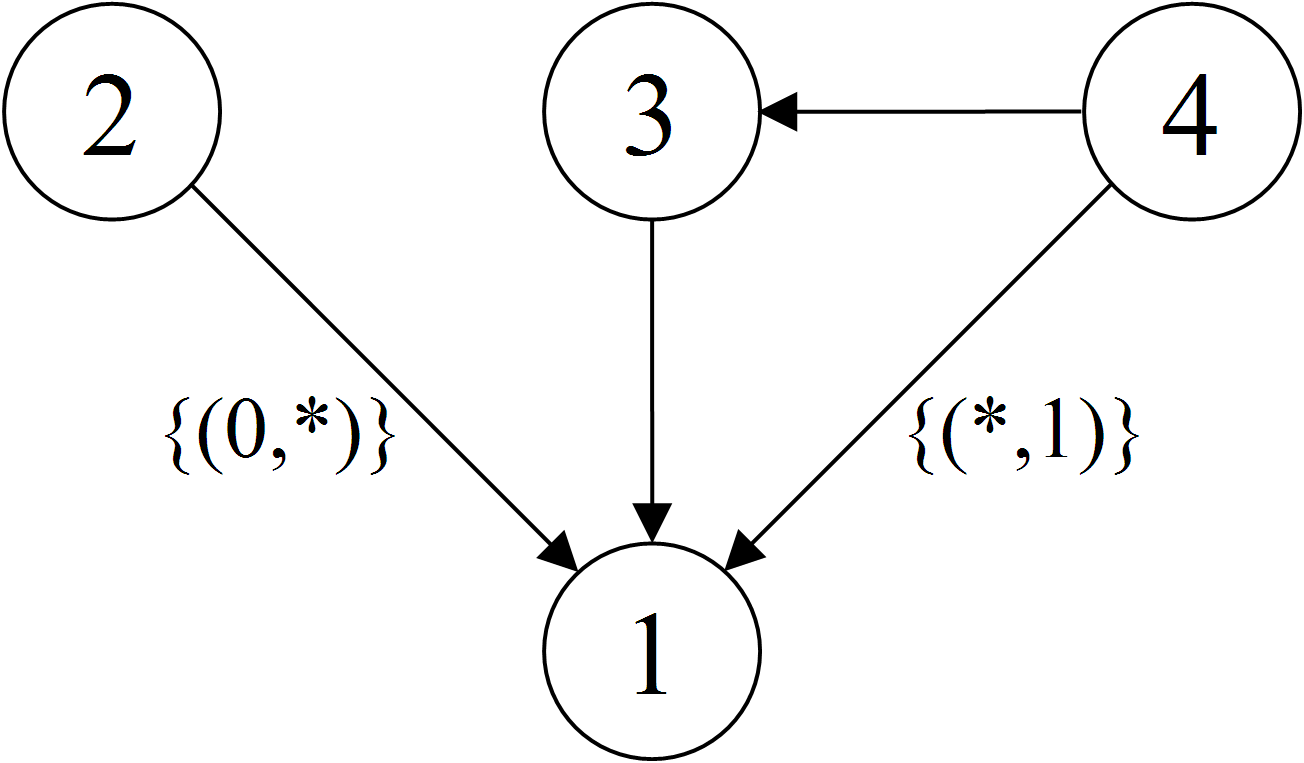}

}~~~~~\subfloat[\label{fig:specialCase2} ]{\includegraphics[height=2.88cm]{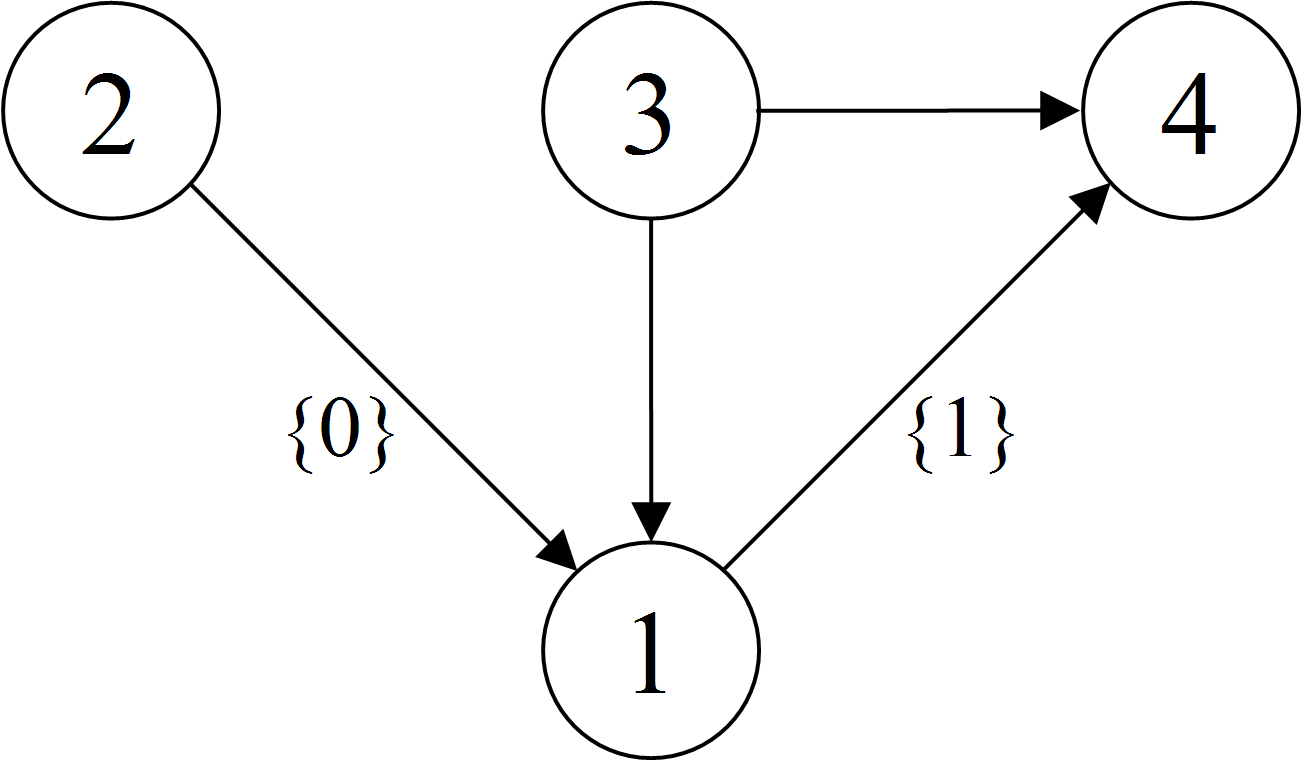}

}
\caption{LDAGs with CI inducing CSI-structures. \label{fig:CSIeqLDAGs}}
\par\end{centering}
\end{figure}
Consider the LDAG in Figure \ref{fig:specialCase1}. When only considering the underlying DAG, it appears that
\[
X_{2}\not\perp X_{4}\mid\mid_{G} X_{1},X_{3}\Rightarrow X_{2}\not\perp X_{4}\mid X_{1},X_{3}
\]
However, through CSI-separation and reasoning by cases we recover
\[
\begin{array}{cc}
X_{2}\perp X_{4}\mid\mid_{G_{L}} X_{1},X_{3}=0 & \Rightarrow X_{2}\perp X_{4}\mid X_{1},X_{3}=0\\
X_{2}\perp X_{4}\mid\mid_{G_{L}} X_{1},X_{3}=1 & \Rightarrow X_{2}\perp X_{4}\mid X_{1},X_{3}=1
\end{array}
\]
which eventually leads us to the conclusion that
\[
X_{2}\perp X_{4}\mid X_{1},x_{3}\ \ \forall x_{3}\in \mathcal{X}_{3}\Leftrightarrow X_{2}\perp X_{4}\mid X_{1},X_{3}.
\]
\emph{d}-separation is based on the notion of active trails, i.e. trails along which information can flow from one variable to another, and a lack of such trails will imply \emph{d}-separation. Labels in an LDAG have the ability to cut off an active trail for a certain context by removing an edge in it and render the trail non-active or blocked in that context. The regularity condition prohibits this from occurring throughout the outcome space for a single edge but certain combinations of labels can still deactivate a trail that appears active when only considering the underlying DAG.

Now consider the LDAG in Figure \ref{fig:specialCase2}. When considering the underlying DAG alone, it appears that
\[
X_{2}\not\perp X_{4}\mid\mid_{G} \varnothing \Rightarrow X_{2}\not\perp X_{4}.
\]
However, we can recover the following CSIs through CSI-separation:
\[
\left. \begin{array}{cc}
X_{2}\perp X_{4}\mid\mid_{G_{L}} X_{3}=0\\ 
X_{2}\perp X_{4}\mid\mid_{G_{L}} X_{3}=1 
\end{array}\right\} \Rightarrow X_{2}\perp X_{4}\mid X_{3}
\]
\[
\begin{array}{cc}
X_{2}\perp X_{3}\mid\mid_{G}\ \varnothing & \Rightarrow X_{2}\perp X_{3}
\end{array}
\]
The first of the CIs must be discovered through reasoning by cases in the same way as in the previous example while the second is easily discovered from the underlying DAG. Combining the CIs leads us indirectly to the conclusion that 
\[
 X_{2}\perp X_{4}
\]
indeed holds due to the structural properties of the LDAG. Several CSI-separation statements work together in order to achieve a non-local independence that is not easily discovered. However, both these situations are special cases that can only arise when the complete outcome space of a subset of variables is split up over several labels.

Earlier we introduced the class of regular maximal LDAGs and concluded that we can restrict the model space to this substantially smaller subclass without loosing any generality. However, in this subclass there still exist large classes of distinct LDAGs that encode equivalent dependence structures. \citet{Heckerman95LearningBN} highlighted a fundamental aspect in that classes of distinct DAGs may determine the same statistical model. Every DAG within such a class will determine the same set of CI restrictions among the variables in the model. \citet{Andersson97MEclass} characterized these so called Markov equivalence (also known as $\mathcal{I}$-equivalence) classes  by concluding that each class corresponds to an essential graph in the form of a chain graph. As for DAGs, the difference between two equivalent LDAGs can occur from reversing non-essential edges. It is worth noting that the criteria for an edge being essential will differ from DAGs. This observation is based on the fact that the direction of the edges determines which local CSIs may be included in an LDAG.  

All DAGs within the same Markov equivalence class share the same dependence structure or $\mathcal{I}(G)$. Correspondingly, we now define CSI-equivalence for LDAGs.
\begin{definition}\textit{CSI-equivalence for LDAGs}\\
Let $G_{L}=(V,E,\mathcal{L}_E)$ and $G_{L}^{*}=(V,E^{*},\mathcal{L}_{E}^{*})$ be two distinct regular maximal LDAGs. The LDAGs are said to be CSI-equivalent if $\mathcal{I}({G_{L}})=\mathcal{I}(G_{L}^{*})$. A set containing all CSI-equivalent LDAGs forms a CSI-equivalence class.
\end{definition}
In the remainder of this section we will discuss some structural properties that two distinct LDAGs must fulfill to belong to the same CSI-equivalence class. We begin by considering the underlying DAG.
\begin{theorem}\label{The-skeleton-condition}
Let $G_{L}=(V,E,\mathcal{L}_{E})$ and $G_{L}^{*}=(V,E^{*},\mathcal{L}_{E}^{*})$ be two regular maximal LDAGs belonging to the same CSI-equivalence class. Their underlying DAGs $G=(V,E)$ and $G^{*}=(V,E^{*})$ must then have the same skeleton. 
\end{theorem}
\begin{proof}
This theorem is a direct consequence of Theorem \ref{regmax edgerem}. $\square$
\end{proof}
Next we introduce a criterion that ties together the concept of CSI-equivalence among LDAGs and the concept of Markov equivalence among DAGs.
\begin{theorem}\label{MEclass}
Let $G_{L}=(V,E,\mathcal{L}_{E})$ and $G_{L}^{*}=(V,E^{*},\mathcal{L}_{E}^{*})$ be two maximal regular LDAGs for which there exists distributions $P$ and $P^{*}$ such that $\mathcal{I}(G_{L})=\mathcal{I}(P)$ and $\mathcal{I}(G_{L}^{*})=\mathcal{I}(P^{*})$. $G_{L}$ and $G_{L}^{*}$ are CSI-equivalent if and only if their corresponding context-specific LDAGs $G_{L}(x_{V})=G(x_{V})$ and $G_{L}^{*}(x_{V})=G^{*}(x_{V})$ are Markov equivalent for all $x_{V} \in \mathcal{X}_{V}$.
\end{theorem}
\begin{proof}
Let $G_{L}=(V,E,\mathcal{L}_{E})$ and $G_{L}^{*}=(V,E^{*},\mathcal{L}_{E}^{*})$ be two maximal regular LDAGs for which there exists distributions $P$ and $P^{*}$ such that $\mathcal{I}(G_{L})=\mathcal{I}(P)$ and $\mathcal{I}(G_{L}^{*})=\mathcal{I}(P^{*})$. 

($\Rightarrow$) Assume that $G_{L}$ and $G_{L}^{*}$ are CSI-equivalent. Assume further that there exists a joint outcome $x_{V}\in \mathcal{X}_{V}$ for which $G(x_{V})$ and $G^{*}(x_{V})$ are not Markov equivalent, i.e. they have different (1) skeletons or (2) immoralities. (1) If they have different skeletons, there exists an edge $\{ i,j\}$ in, say, the skeleton of $G(x_{V})$ that does not exist in the skeleton of $G^{*}(x_{V})$. Due to Theorem \ref{The-skeleton-condition}, the underlying DAGs must have the same skeleton. The lack of the edge in $G^{*}(x_{V})$ implies that a local CSI of the form $X_{j}\perp X_{i}\mid x_{L_{(i,j)}}$ holds in $G_{L}^{*}$ but not in $G_{L}$. (2) Assume that there exists an immorality $i \rightarrow j \leftarrow k$ in, say, $G(x_{V})$ that does not exist in $G^{*}(x_{V})$. If there does not exist an edge between $i$ and $k$ in $G_{L}$ (and $G_{L}^{*}$), there must exist some $S\subseteq V\setminus \{i,j,k\}$ for which $X_{i}\perp X_{k}\mid\mid_{G}X_{S}$ and consequently $\{X_{i}\perp X_{k}\mid X_{S}\}\in\mathcal{I}(G_{L})$ while $\{X_{i}\perp X_{k}\mid X_{S}\} \not\in\mathcal{I}(G_{L}^{*})$ since there exists at least one active trail between $X_{i}$ and $X_{k}$ via $X_{j}$. If there exists an edge between $i$ and $k$ in $G_{L}$ (and $G_{L}^{*}$), there must exist a local CSI of the form $X_{i}\perp X_{k}\mid x_{L_{(k,i)}}$ (or $X_{k}\perp X_{i}\mid x_{L_{(i,k)}}$) that holds in $G_{L}^{*}$ but not in $G_{L}$ since $j\in L_{(k,i)}$ (or $j\in L_{(i,k)}$) in $G^{*}_{L}$ while $j\not\in L_{(k,i)}$ (and $j\not\in L_{(i,k)}$) in $G_{L}$. (1) and (2) allow us to conclude that $\mathcal{I}({G_{L}})\not=\mathcal{I}(G_{L}^{*})$ which contradicts the assumption of $G_{L}$ and $G_{L}^{*}$ being CSI-equivalent. $G_{L}(x_{V})$ and $G_{L}^{*}(x_{V})$ must be Markov equivalent for all $x_{V} \in \mathcal{X}_{V}$. 

($\Leftarrow$) Assume that $G_{L}(x_{V})$ and $G_{L}^{*}(x_{V})$ are Markov equivalent for all $x_{V} \in \mathcal{X}_{V}$. Let $P$ be a distribution for which $G_{L}$ is a perfect CSI-map. Each joint probability $p(X_{V}=x_{V})$ factorizes according to $G_{L}(x_{V})$. Since $G_{L}(x_{V})$ and $G_{L}^{*}(x_{V})$ are Markov equivalent, we can refactorize each joint probability $p(X_{V}=x_{V})$ according to $G_{L}^{*}(x_{V})$ without altering the joint distribution or inducing any additional dependencies. This means that $G_{L}^{*}$ is also a perfect CSI-map of $P$. Since $\mathcal{I}({G_{L}})=\mathcal{I}(P)=\mathcal{I}(G_{L}^{*})$, we can conclude that $G_{L}$ and $G_{L}^{*}$ are CSI-equivalent. $\square$
\end{proof}
From this theorem it is clear that the LDAGs in Figure \ref{fig:CSIeqLDAGs} are indeed CSI-equivalent even if some of the independencies are not so obvious. In order to check CSI-equivalence between two LDAGs, it suffices to compare context-specific graphs for only a subset of variables since not all will affect the structure of the graphs. Furthermore, all outcomes for which no labels in either graph are satisfied, need only to be checked once as the context-specific graphs in all these cases are equal to the underlying DAG. This last observation leads to a more strict version of Theorem \ref{The-skeleton-condition} given that a specific condition is satisfied.
\begin{theorem}
Let $G_{L}=(V,E,\mathcal{L}_{E})$ and $G_{L}^{*}=(V,E^{*},\mathcal{L}_{E}^{*})$ be two regular maximal LDAGs that are CSI-equivalent and let their labelings be such that there exists at least one joint outcome, $x_{V}\in\mathcal{X}_{V}$, for which no label is satisfied. Their underlying DAGs $G=(V,E)$ and $G^{*}=(V,E^{*})$ must then be Markov equivalent.
\end{theorem}
\begin{proof}
This theorem is a direct consequence of Theorem \ref{MEclass}. $\square$
\end{proof}

\section{Bayesian learning of LDAGs by non-reversible MCMC\label{sec:Bayesian-learning-of-LDAGs}}
This section will attend the intricate problem of learning the LDAG structure from a set of data. This poses some obvious problems due to the extremely vast model space as well as some additional not so obvious problems due to the flexibility of the models. We introduce a structural learning method that utilizes a non-reversible Markov Chain Monte Carlo (MCMC) method combined with greedy hill climbing. Such a combination of a stochastic and a deterministic algorithm provides solid performance with a reasonable time complexity. A Bayesian score is used to evaluate the appropriateness of an LDAG given a set of observed data. In order to prevent overfitting, we impose a prior distribution that allows us to balance the ability of an LDAG to match the available learning data with its complexity. We begin with some additional notations.

Let $\mathbf{X}=\{\mathbf{x}_{i}\}_{i=1}^{n}$ denote a set of training data consisting of $n$ observations $\mathbf{x}_{i}=(x_{i1},\ldots x_{id})$ of the variables $\{X_{1},\ldots,X_{d}\}$ such that $\mathbf{x}_{i}\in\mathcal{X}$. We assume that $\mathbf{X}$ is complete in the sense that it contains no missing values. We denote an LDAG by $G_{L}$ and $\mathcal{G}_{L}$ denotes the set of all regular maximal LDAGs. We let $\Theta_{G_{L}}$ denote the parameter space induced by an LDAG and $dim(\Theta_{G_{L}})$ denotes the number of free parameters spanning the parameter space. An instance $\theta\in\Theta_{G_{L}}$ corresponds to a specific joint distribution that factorizes according to the LDAG $G_{L}$. The CSI-consistent partition of the outcome space $\mathcal{X}_{\Pi_{j}}$ is denoted by $\mathcal{S}_{\Pi_{j}}=\{S_{j1},\ldots,S_{jk_{j}}\}$ where $k_{j}=|\mathcal{S}_{\Pi_{j}}|$ is the number of outcome classes. We let $r_{j}=|\mathcal{X}_{j}|$ and $q_{j}=|\mathcal{X}_{\Pi_{j}}|$ denote the cardinality of the outcome space of variable $X_{j}$ and its parents $X_{\Pi_{j}}$, respectively. Finally, we use $n(x_{ij}\times S_{jl})$ to denote the total count of the configurations $\{x_{ij}\}\times S_{jl}$ in $\mathbf{X}$.

In the Bayesian approach to model learning, one considers the posterior distribution of the models given some data,  
\begin{equation}
p(G_{L}\mid\mathbf{X})=\frac{p(\mathbf{X}\mid G_{L})\cdot p(G_{L})}{\underset{G_{L}\in\mathcal{G}_{L}}{\sum}p(\mathbf{X}\mid G_{L})\cdot p(G_{L})}.\label{eq:postdistr}
\end{equation}
Here $p(\mathbf{X}\mid G_{L})$ is the marginal probability of observing the data $\mathbf{X}$ (evidence) given a specific LDAG $G_{L}$ and $p(G_{L})$ denotes the prior probability of the LDAG. The denominator is a normalizing constant that does not depend on $G_{L}$ and it can be ignored for the purpose of comparing particular graphs. Our main interest is to find the maximum a posteriori model, i.e. the solution to
\begin{equation}
\arg\underset{{\scriptstyle G_{L}\in\mathcal{G}_{L}}}{\max}p(\mathbf{X}\mid G_{L})\cdot p(G_{L}).\label{eq:argmax}
\end{equation}
To evaluate $p(\mathbf{X}\mid G_{L})$, we need to consider all possible instances of the parameter vector satisfying the independencies encoded by the LDAG and weight them with respect to a prior according to 
\begin{equation}
p(\mathbf{X}\mid G_{L})=\underset{{\scriptstyle \theta\in\Theta_{G_{L}}}}{\int}p(\mathbf{X}\mid G_{L},\theta)\cdot f(\theta\mid G_{L})d\theta,\label{eq:MLintegral}
\end{equation}
where $p(\mathbf{X}\mid G_{L},\theta)$ and $f(\theta\mid G_{L})$ are the respective likelihood function and prior distribution over the parameters, given the graph $G_{L}$.

Under certain assumptions, (\ref{eq:MLintegral}) can be solved analytically for DAGs, see \citet{CH92IndProbNet} and \citet{Buntine91TheoryRef}. \citet{Heckerman95LearningBN} identify and discuss the assumptions in detail. \citet{Friedman96learnBNlocstruct} and \citet{Chickering97abayesian} derive a corresponding closed-form expression for structures based on CPT-trees and decision graphs, respectively. Since an LDAG induces partitionings of the parental outcome spaces in a similar manner as these previous works, the marginal likelihood of an LDAG can be expressed as  
\begin{equation}
p(\mathbf{X}\mid G_{L})=\overset{{\scriptstyle d}}{\underset{{\scriptstyle j=1}}{\prod}}\overset{{\scriptstyle k_{j}}}{\underset{{\scriptstyle l=1}}{\prod}}{ \frac{\Gamma\left(\sum_{i=1}^{r_{j}}\alpha_{ijl}\right)}{\Gamma\left(n(S_{jl})+\sum_{i=1}^{r_{j}}\alpha_{ijl}\right)}}\overset{{\scriptstyle r_{j}}}{\underset{{\scriptstyle i=1}}{\prod}}\frac{{ \Gamma\left(n(x_{ji}\times S_{jl})+\alpha_{ijl}\right)}}{{ \Gamma\left(\alpha_{ijl}\right)}},\label{eq:ML LDAG}
\end{equation}
where $n(\cdot)$ is the count defined earlier and the $\alpha_{ijl}$:s are hyperparameters (also known as pseudocounts) defining a collection of local Dirichlet distributions. The hyperparameters  characterize our prior belief about the CPDs and must be established to evaluate (\ref{eq:ML LDAG}). \citet{Buntine91TheoryRef} defines a non-informative prior for ordinary Bayesian networks. As each joint outcome is equally likely for this prior, it ensures that equivalent networks are evaluated equally by the marginal likelihood. Under some additional assumptions, \citet{Heckerman95LearningBN} showed that likelihood equivalence can also be achieved by deriving each $\alpha_{ijl}$ from a prior network. The priors discussed in \citet{Friedman96learnBNlocstruct} and \citet{Chickering97abayesian} extend this idea to structures based on compact representations of the CPDs. We define our prior by setting
\begin{equation}
\alpha_{ijl}=\frac{N}{r_{j}\cdot q_{j}}\cdot|S_{jl}|,\label{eq:hyppara}
\end{equation}
where $q_{j}$ is with respect to the underlying DAG and $|S_{jl}|$ denotes the number of configurations in that specific part. This non-informative prior can thus be considered an extension of the one used in \citet{Buntine91TheoryRef} which in turn is a special case of \citet{Heckerman95LearningBN}. The parameter $N$, known as the equivalent sample size, reflects the strength of our prior belief on the parameter distributions. Its effect on the choice of Bayesian network structures has been investigated by \citet{Silander07ESS}. 

The only remaining issue at this point is to define the prior distribution over the set of LDAGs. This part of (\ref{eq:postdistr}) is generally not given too much attention in Bayesian model learning but for LDAGs it plays a vital role. A common approach is to assume a uniform prior and simply base the scoring function on the marginal likelihood alone. A uniform prior has been shown to work quite well for ordinary DAGs. \citet{Chickering97abayesian} use this prior as their main focus is to maximize the marginal likelihood rather than looking at criteria such as predictive performance or structural differences from a generative model. However, they also propose another prior that penalizes complexity of a model in terms of the number of free parameters,
\begin{equation}
p(G_{L})\propto \kappa^{dim(\Theta_{G_{L}})}=\overset{{\scriptstyle d}}{\underset{{\scriptstyle j=1}}{\prod}}\kappa^{dim(\Theta_{G_{L}}(j))}\label{eq:priorParaPen}
\end{equation}
where $\kappa \in (0,1]$.This approach is more in line with the one used by \citet{Friedman96learnBNlocstruct} who use Kullback-Leibler divergence to further analyze the models chosen according to the scoring function.

The choice of model prior turns out to be an essential part of the Bayesian scoring function for LDAGs. We show in the result section that the marginal likelihood alone has a tendency to overfit the dependence structure for limited sample sizes by favoring dense graphs with complex labelings. The number of free parameters associated with such a LDAG is low compared to the number of free parameters associated with its underlying DAG and the LDAG is said to have a high CSI-complexity. The overfitting effect is thus reflected through a high CSI-complexity rather than an excessive number of free parameters. Although high CSI-complexity models may lead to high marginal likelihoods, they are more prone to contain false dependencies and thereby fail to capture the true global dependence structure. This has a direct negative effect on their out-of-sample predictive performance. Another drawback is that their high density will yield bulky CPDs in (\ref{eq:facDAG}). This basically counteracts the fundamental idea of modularity on which the concept of graphical models is based.   

The overfitting phenomenon vanishes asymptotically when $n\rightarrow\infty$, since maximization of the marginal likelihood leads to a consistent estimator of the model structure. Consequently, we construct our prior such that it acts as a regularizer for smaller sample sizes and its effect will gradually vanish as the sample size is increased,
\begin{equation}
p(G_{L})\propto\kappa^{dim(\Theta_{G})-dim(\Theta_{G_{L}})}=\overset{{\scriptstyle d}}{\underset{{\scriptstyle j=1}}{\prod}}\kappa^{dim(\Theta_{G(j)})-dim(\Theta_{G_{L}(j)})}\label{eq:priorLDAG},
\end{equation}
where $dim(\Theta_{G_{L}})$ and $dim(\Theta_{G})$ are the number of free parameters associated with the LDAG and its underlying DAG, respectively. The parameter $\kappa\in(0,1]$ can be considered a measure of how strongly a CSI inducing label configuration must be supported by the data in order for it to be included in the model. For small values on $\kappa$, addition of a label configuration increases the score only if its associated CSI is firmly supported by the data while $\kappa=1$ corresponds to a uniform prior. This prior is similar to (\ref{eq:priorParaPen}) but with the important distinction that the penalty degree is now determined by CSI-complexity rather than complexity in terms of number of free parameters which is implicitly restrained by the marginal likelihood. Regardless of the structure of the underlying DAG, all LDAGs with the same amount of label induced CSIs will thus have the same prior probability. This is motivated by the fact that we do not know the true global dependence structure, i.e. the underlying DAG. Instead we adjust our prior belief in how high degree of CSI-complexity the data is able to faithfully express without imposing false dependencies.

Our prior shares some desirable properties with the marginal likelihood (\ref{eq:ML LDAG}). Given Markov equivalent underlying DAGs, all CSI-equivalent LDAGs are evaluated equally. When considering two equivalent LDAGs with non-equivalent underlying DAGs, the prior will favor the one with lower CSI-complexity. This is, however, the one to be preferred as it has a simpler interpretation. Another important property of (\ref{eq:priorLDAG}) is that it decomposes variable-wise. From a computational perspective, this greatly enhances the efficiency of the search algorithm introduced later. On the downside, an unavoidable issue with an adjustable prior (or regularizer) is the task of determining the optimal value of some tuning parameter (in our case $\kappa$). In the end of this section we propose a cross-validation-based method which allows us to choose among several values on $\kappa$ before the actual model learning. 

\begin{algorithm}
\textbf{Procedure} Optimize-Local-Structure(

\hspace{1cm}$X_{j}$,\hspace{1cm}//\emph{Variable whose local structure is optimized}

\hspace{1cm}$X_{\Pi_{j}}$,\hspace{0.76cm}//\emph{Parental variables}

\hspace{1cm}$\mathbf{X}$,\hspace{1.08cm}//\emph{A set of complete data over $X_{\Pi_{j} \cup \{ j\}}$}

\hspace{1cm})

1:\ \ $\mathcal{L}_{j}=\{ \mathcal{L}_{(i,j)} \}_{i \in \Pi_{j}} \leftarrow \varnothing$

2:\ \ $keepClimb \leftarrow True$

3:\ \ \textbf{while} $keepClimb$

4: \hspace{0.5cm}\ \ $\mathcal{L}_{j}^{top}\leftarrow \mathcal{L}_{j}$

5: \hspace{0.5cm}\ \ \textbf{for} $x_{L_{(i,j)}}\notin\mathcal{L}_{j}:\{x_{L_{(i,j)}}\cup \mathcal{L}_{(i,j)}\} \subset \mathcal{X}_{L_{(i,j)}}$

6: \hspace{1cm}\ \ $\mathcal{L}_{j}^{cand} \leftarrow \mathcal{L}_{j}\cup x_{L_{(i,j)}}$

7: \hspace{1cm}\ \ \textbf{if} $p(\mathbf{X}_{ j}\mid \mathbf{X}_{\Pi_{j}}, \mathcal{L}_{j}^{cand})>p(\mathbf{X}_{ j}\mid \mathbf{X}_{\Pi_{j}}, \mathcal{L}_{j}^{top})$

8: \hspace{1.5cm}\ \ $\mathcal{L}_{j}^{top} \leftarrow \mathcal{L}_{j}^{cand}$

9: \hspace{1cm}\ \ \textbf{end}

10: \hspace{0.5cm} \textbf{end}

11: \hspace{0.5cm} \textbf{if} $\mathcal{L}_{j} < \mathcal{L}_{j}^{top}$

12: \hspace{1cm} $\mathcal{L}_{j} \leftarrow \mathcal{L}_{j}^{top}$

13: \hspace{1cm} $\mathcal{L}_{j} \leftarrow makeMaximal(\mathcal{L}_{j})$

14: \hspace{0.5cm} \textbf{else}

15: \hspace{1cm} $keepClimb \leftarrow False$

16: \hspace{0.5cm} \textbf{end}

17: \textbf{end}

18: \textbf{Return} $\mathcal{L}_{j}$

\caption{Procedure for optimizing the local CSI structure for $X_{j}$\label{alg:Procedure-for-optimizing}}
\end{algorithm}
Given a scoring function, the task of learning an LDAG structure reduces to finding the model that maximizes the score given the data. This is, however, a very challenging problem since the model space is enormous. The number of DAGs for $d$ variables grows super-exponentially with $d$ (\citet{key-13}) . In practice it is hence infeasible to calculate the posterior distribution (\ref{eq:postdistr}) even for a small number of variables. Furthermore, this only covers ordinary DAGs and an expansion of the model space to include LDAGs will further increase the intractability of an exhaustive evaluation. Consequently, we need to apply some form of a search method. For this purpose we introduce a search algorithm which utilizes a non-reversible MCMC method, introduced and discussed by \citet{key-10,key-19}, combined with a direct form of optimization. The general idea is that the stochastic part of the algorithm jumps between neighbouring underlying DAGs, whose CSI structures are optimized by adding labels in a \textit{greedy hill climbing}-manner. As our score decomposes variable-wise, instead of considering the whole DAG, we can optimize the local structure of one variable at a time. The procedure is described in Algorithm \ref{alg:Procedure-for-optimizing}. For the score derived in the previous section, the termination of the algorithm occurs when the improvement of the marginal likelihood falls below the predetermined value of $\kappa$. Any deterministic optimization strategy similar to Algorithm \ref{alg:Procedure-for-optimizing} basically maps the set of DAGs onto a subset of regular maximal LDAGs $\mathcal{G}_{L}^{opt}\subseteq\mathcal{G}_{L}$. This will bring down the size of the model space explored by our MCMC method to the number of DAGs.

Various forms of MCMC are generally proposed for the Bayesian approach for learning the structural layer of probabilistic models. We utilize a non-reversible version which has been shown to possess several advantageous properties (\citet{key-10,key-19}). Let $q(\cdot|G_{L})$ denote a generic proposal distribution over the model space $\mathcal{G}_{L}^{opt}$, given $G_{L}$ for all $G_{L}\in\mathcal{G}_{L}^{opt}$. We let $G_{L}(t)$ denote the state of the chain at iteration $t$. At iteration $t=1,2,\ldots$ of the non-reversible chain, $q(\cdot|G_{L}(t))$ is used to generate the next candidate state $G_{L}^{*}$ which is then accepted with probability
\[
\min\left(1,\frac{p(G_{L}^{*})p(\mathbf{X}|G_{L}^{*})}{p(G_{L}(t))p(\mathbf{X}|G_{L}(t))}\right).
\]
If $G_{L}^{*}$ is accepted, we set $G_{L}(t+1)=G_{L}^{*}$ and otherwise $G_{L}(t+1)=G_{L}(t)$. The proposal probabilities need not to be explicitly calculated or even known as long as they remain unchanged over the iterations and the resulting chain is irreducible. The stationary distribution of such a chain does no longer follow the posterior distribution (\ref{eq:postdistr}). However, our main objective is, as previously stated, to identify only the maximum a posteriori model (\ref{eq:argmax}). The approximate solution proposed by a search chain at iteration $t$ is simply the one with the highest score visited thus far. Satisfying the conditions mentioned, the proposal distributions are defined as uniform distributions over the globally adjacent LDAGs that can be reached by adding, reversing or removing a single edge under the restriction that the resulting LDAG is acyclic.

As the difference between two successive graphs may only differ for a single edge, at most two local structures are modified at each step of the chain. Since our score $p(\mathbf{X},G_{L})$ decomposes variable-wise, only the modified local structures must be re-evaluated as the score for the rest of the variables remains unchanged. This idea can be further exploited when optimizing the local CSI-structures. At each step of the optimization procedure, we need only to re-evaluate the score with respect to the parts of the partition that are modified. For our algorithm in particular, only a single new part is created for each added label configuration.

Adding of labels yields a flexibility that facilitates the identification of "weaker" edges that might be deemed non-existing in the model space of DAGs. However, optimization of the CSI-structure cannot make up for unrealistic global independence assumptions made by an inferior underlying DAG structure. Hence, a prerequisite for learning a good LDAG structure is that it is based on a sensible underlying DAG. Getting stuck at regions with inferior underlying DAGs, will have a more severe negative effect on the learned LDAGs than not finding the optimal CSI-structure. This motivates the fact that the stochastic part of our method performs global changes whereas the optimization of the CSI-structures is done in a deterministic manner. 

To finally attend the problem of choosing an appropriate value of $\kappa$, we propose a cross-validation scheme that allows us to assess a set of candidate values. First we partition the data $\mathbf{X}$ into a training set $\mathbf{Y}$ and a test set $\mathbf{Z}$. We then apply our search method on the training data under some prior (or $\kappa$) and identify the optimal model $G_{L}^{\kappa}$. We then evaluate the learned model's ability to predict the test data by calculating the posterior predictive probability of the test data given the training data,
\begin{equation}
p(\mathbf{Z}\mid \mathbf{Y},G_{L}^{\kappa})=\underset{{\scriptstyle \theta\in\Theta_{G_{L}^{\kappa}}}}{\int}p(\mathbf{Z}\mid G_{L}^{\kappa},\theta)\cdot f(\theta\mid \mathbf{Y},G_{L}^{\kappa})d\theta.\label{eq:PREDintegral}
\end{equation}
This integral is similar to (\ref{eq:MLintegral}) but the parameter vectors are now weighted with respect to the posterior distributions updated according to the training data. Under similar assumptions made earlier, (\ref{eq:PREDintegral}) can be calculated analytically by
\begin{align}
  \label{eq:PREDsol}
  \begin{split}
   p(\mathbf{Z}\mid\mathbf{Y},G_{L}^{\kappa})  = \ &\overset{{d}}{\underset{{j=1}}{\prod}}\overset{{k_{j}}}{\underset{{l=1}}{\prod}}{ \frac{\Gamma\left(\sum_{i=1}^{r_{j}}(\alpha_{ijl}+n_{\mathbf{Y}}(x_{ji}\times S_{jl}))\right)}{\Gamma\left(n_{\mathbf{Z}}(S_{jl})+\sum_{i=1}^{r_{j}}(\alpha_{ijl}+n_{\mathbf{Y}}(x_{ji}\times S_{jl}))\right)}}\cdot\\
    &\overset{{r_{j}}}{\underset{{i=1}}{\prod}}\frac{{ \Gamma\left(n_{\mathbf{Z}}(x_{ji}\times S_{jl})+\alpha_{ijl}+n_{\mathbf{Y}}(x_{ji}\times S_{jl})\right)}}{{ \Gamma\left(\alpha_{ijl}+n_{\mathbf{Y}}(x_{ji}\times S_{jl})\right)}},
  \end{split}
\end{align}
where the bold case index indicates to which data set the outcome count refers. To reduce the variability of the method, multiple partitions  of $\mathbf{X}$ are created, $\{(\mathbf{Y}_{1},\mathbf{Z}_{1}),(\mathbf{Y}_{2},\mathbf{Z}_{2}),\ldots,(\mathbf{Y}_{M},\mathbf{Z}_{M})\}$, and the validation results are averaged according to
\begin{equation}
\rho_{pred}(\kappa)=\frac{1}{M}\overset{M}{\underset{{\scriptstyle m=1}}{\sum}}\log p(\mathbf{Z}_{m}\mid\mathbf{Y}_{m},G_{L}^{\kappa,m}).  \label{eq:rhopred}
\end{equation}
The value on $\kappa$ is finally chosen among the candidates as the one that maximizes (\ref{eq:rhopred}).
\section{Experimental results with real and simulated data sets\label{sec:Examples}}
To illustrate the properties of LDAGs, we apply our search algorithm on both a real and two simulated data sets. First we consider a real
data set that has been thoroughly investigated in earlier graphical modelling literature. After that we consider synthetic DAG- and LDAG-based models. Throughout this section we set the equivalent sample size $N=1$. By keeping $N$ constant, we instead focus on investigating how different values on $\kappa$ will affect the ability of the learned graph to predict the data structure. The learning algorithm was executed for $\kappa\in\{0.001,0.1,0.3,0.5\}$. The optimal value on $\kappa$ is chosen according to the cross-validation scheme described earlier. To create multiple partitions, the data is split up in to ten parts of the same size and each part is successively chosen as test set. For each search, we initiated 50 parallel independent search chains. The empty graph was set as the initial state of each chain and number of iterations was set to 500. The optimal graph was then simply identified as the one with the highest score visited by any of the chains.

Our real data set contains 1841 cases composed of six binary risk factors for coronary heart disease (\citet{key-4}). The meanings of the variables are explained in Table \ref{tab:Explanation heart data} (Appendix). In Table \ref{tab:PropertiesHeatLDAGs} the structural properties of the graphs, identified for different values on $\kappa$, are listed along with their scores. Here we get an indication of how the CSI-complexity increases with higher values on $\kappa$. The bold font indicate which $\kappa$ was chosen as optimal by the cross-validation procedure. The corresponding LDAG is illustrated in Figure \ref{fig:heartLDAG}. The LDAG identified for $\kappa=0.001$ contains no labels and is thereby equal to its underlying DAG. The improvement, that an added label configuration induces to the marginal likelihood, is overshadowed by the simultaneous lowering of the prior probability mass. Consequently, when $\kappa\rightarrow0$ the direct optimization will map the set of DAGs onto itself and the learning procedure is reduced to a search among ordinary DAGs. 
\begin{table}
\begin{centering}
\begin{tabular}{cccccc}
\hline 
{\footnotesize $\kappa$} & {\footnotesize $\log p(\mathbf{X},G_{L})$} & {\footnotesize $|E|$} & {\footnotesize $dim(\Theta_{G})$} & {\footnotesize $dim(\Theta_{G_{L}})$} & {$\rho_{pred}$}\tabularnewline
\hline 
\hline
{\footnotesize 0.001} & {\footnotesize -6731.82} & {\footnotesize 5} & {\footnotesize 12} & {\footnotesize 12} & {\footnotesize -671.30}\tabularnewline
{\footnotesize 0.1} & {\footnotesize -6729.69} & {\footnotesize 6} & {\footnotesize 14} & {\footnotesize 12} & {\footnotesize -670.88}\tabularnewline
{\textbf{\footnotesize 0.3}} & {\textbf{\footnotesize -6727.50}} & {\textbf{\footnotesize 6}} & {\textbf{\footnotesize 14}} & {\textbf{\footnotesize 12}} & {\textbf{\footnotesize -670.51}}\tabularnewline
{\footnotesize 0.5} & {\footnotesize -6724.68} & {\footnotesize 7} & {\footnotesize 18} & {\footnotesize 11} & {\footnotesize -670.89}\tabularnewline
\hline 
\end{tabular}
\par\end{centering}
\caption{Properties of identified LDAGs for coronary heart disease data. \label{tab:PropertiesHeatLDAGs}}
\end{table}
\begin{figure}
\begin{centering}
\includegraphics[width=5cm]{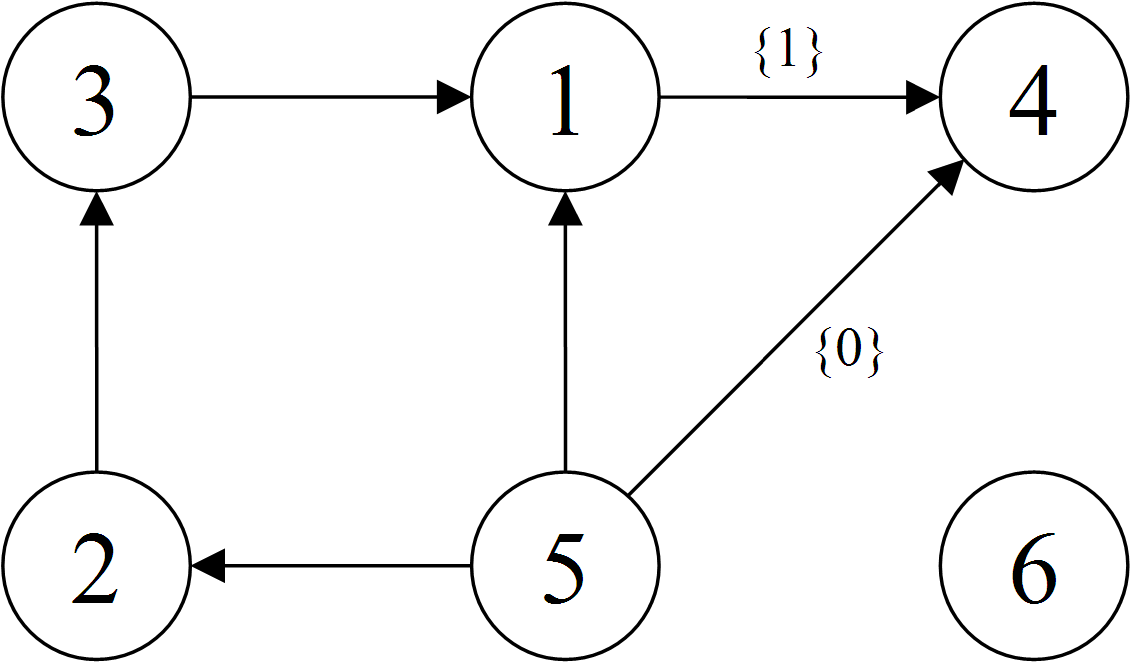}
\caption{Optimal LDAG for coronary heart disease data. \label{fig:heartLDAG}}
\par\end{centering}
\end{figure}

We now consider synthetic models from which data are generated to systematically compare models identified for different prior distributions and sample sizes. Since we know the generating model, we investigate how well the identified models approximate the true distribution. The CPDs of the models are estimated by the consistent mean a posteriori estimator as the expected value of the local posterior Dirichlet distributions. To compare the distributions, we utilize the concept of Kullback-Leibler (KL) divergence. Let $p$ denote the real distribution over $X=(X_{1},\ldots,X_{d})$ and let $p^{*}$ denote an approximation of $p$. The KL divergence between the distributions is defined by
\[
D_{KL}(p\parallel p^{*})=\underset{{\scriptstyle x\in\mathcal{X}}}{\sum}p(x)\log\frac{p(x)}{p^{*}(x)}.
\]
The KL divergence is a non-negative non-symmetric measure of the distance that is equal to zero only if $p=p^{*}$. Under the assumption that no incorrect independence assumptions are made by the model, then $D_{KL}(p\parallel p^{*})\rightarrow0$ as the sample size $n\rightarrow\infty$. In addition to the KL divergence, we also investigate how well the identified LDAGs capture the global dependence structure, i.e. the underlying DAG.

We generated data according to a DAG as well as an LDAG. The LDAG was created by adding labels to the initial DAG. The DAG and labels are illustrated in Figure \ref{fig:locCSIstructWstat} (Appendix). To generate data according to the DAG, each CPD were randomly drawn from a uniform distribution. Similar CPDs may have arisen by chance but no local CSIs were explicitly included. To generate data according to the LDAG, some of the CPDs were set identical in order to satisfy the labels. We let the sample size $n$ range from $250$ to $8000$. For each data set, we executed the learning procedure described earlier. Our results are summarized in Table \ref{PropDAGdata} (Appendix) for the DAG model and in Table \ref{PropLDAGdata} (Appendix) for the LDAG model.

As expected, the model distributions approach the true distribution when the sample size increases. This results in a steady improvement of the KL divergence as illustrated in Figure \ref{fig:allCurves}. The decrease is evident for all values on $\kappa$ but our results indicate that different prior distributions are to be preferred depending on the sample size. It also clear how the quality of most of the models begin to suffer under $\kappa=0.5$ as a result of overfitting. The reduced out-sample-performance of the models prevents the prior from being picked even once. During the simulations it became evident that the overfitting effect further escalated under even less restrictive priors. On the whole, our procedure for picking the optimal $\kappa$ performs well. The thick black curve in Figure \ref{fig:allCurves} represents the models chosen by the cross-validation. Ideally, this curve should stay below the other curves. For the DAG-based model (Figure \ref{fig:DAGmodelAC}), the curve never diverges too far from the optimal choice. For the LDAG-based model (Figure \ref{fig:LDAGmodelAC}), our procedure performs very well by always picking the optimal prior from the candidate set. 
\begin{figure}
\begin{centering}
\subfloat[DAG-based generating model.\label{fig:DAGmodelAC}]{\includegraphics[width=5.7cm]{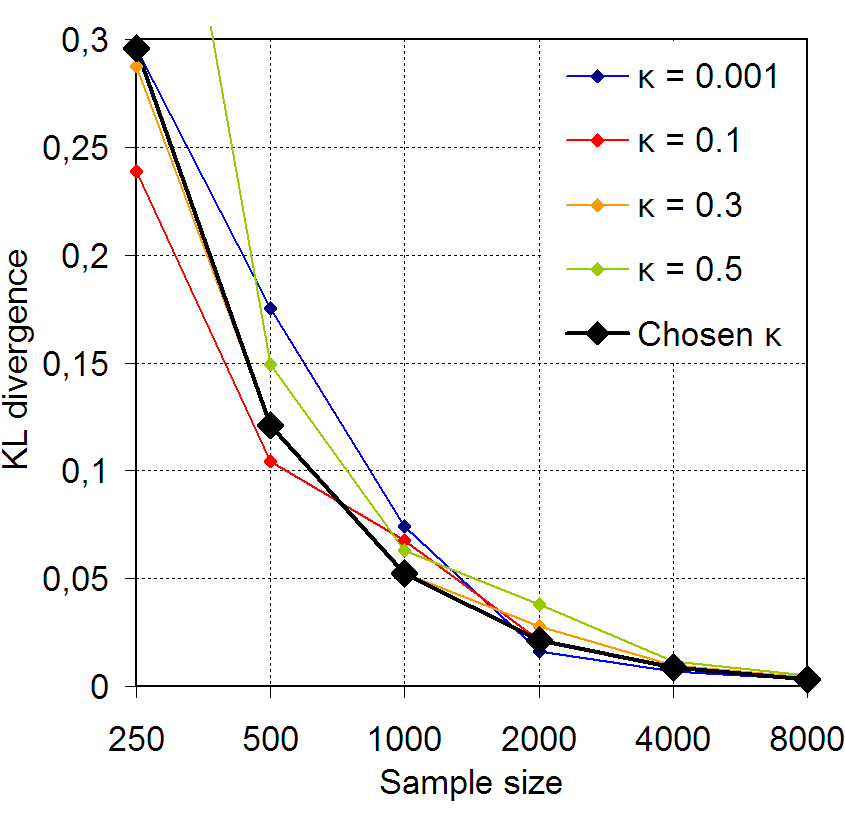}
\hspace{-0.5cm}
}~~~~~\subfloat[LDAG-based generating model. \label{fig:LDAGmodelAC} ]{\includegraphics[width=5.7cm]{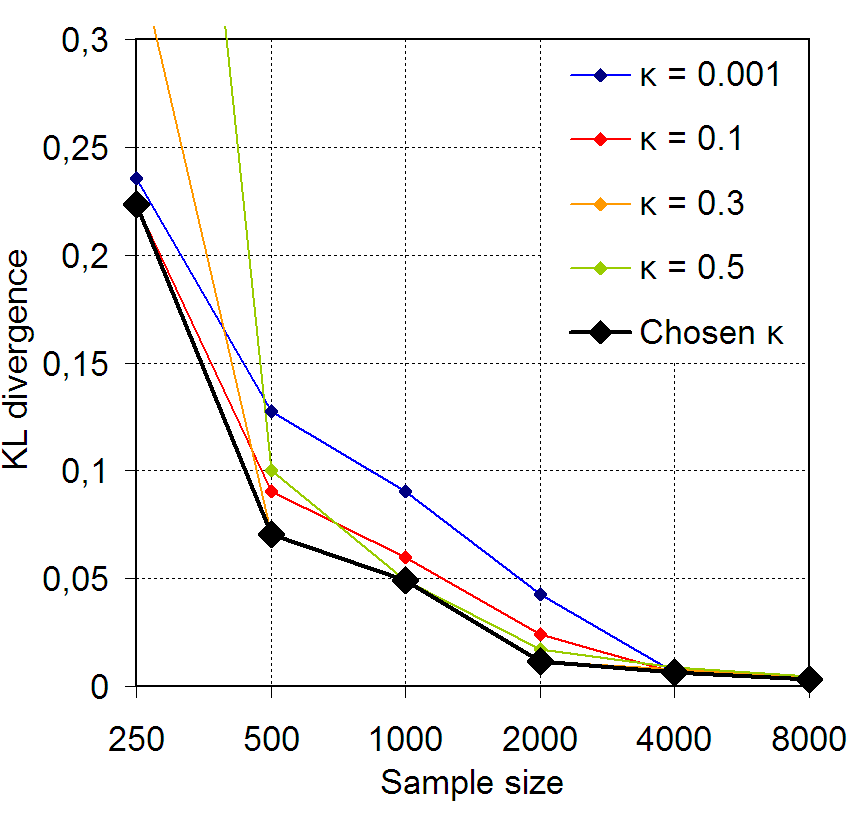}

}
\caption{KL divergence for different sample sizes under different priors. \label{fig:allCurves}}
\par\end{centering}
\end{figure}

As we can see from Table \ref{PropDAGdata} and \ref{PropLDAGdata}, all models identified under $\kappa=0.001$ are without labels since $dim(\Theta_{G})-dim(\Theta_{G_{L}})=0$. We can thus use this prior as a reference point for investigating how well LDAGs perform compared to traditional DAGs.
\begin{figure}
\begin{centering}
\subfloat[DAG-based generating model.\label{fig:DAGmodel}]{\includegraphics[width=5.7cm]{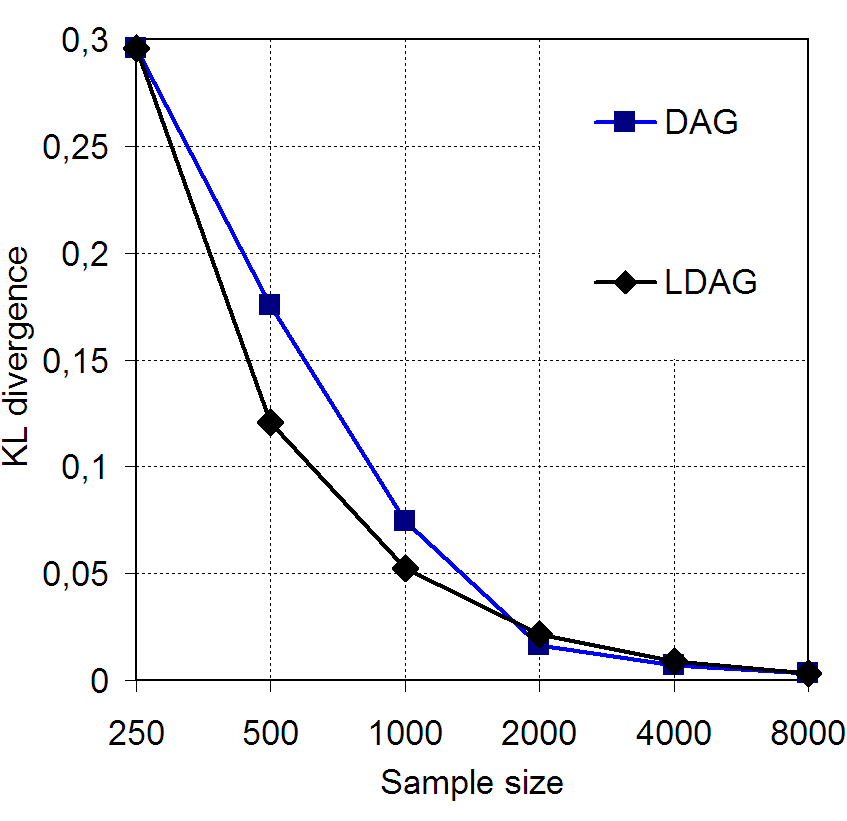}
\hspace{-0.5cm}
}~~~~~\subfloat[LDAG-based generating model. \label{fig:LDAGmodel} ]{\includegraphics[width=5.7cm]{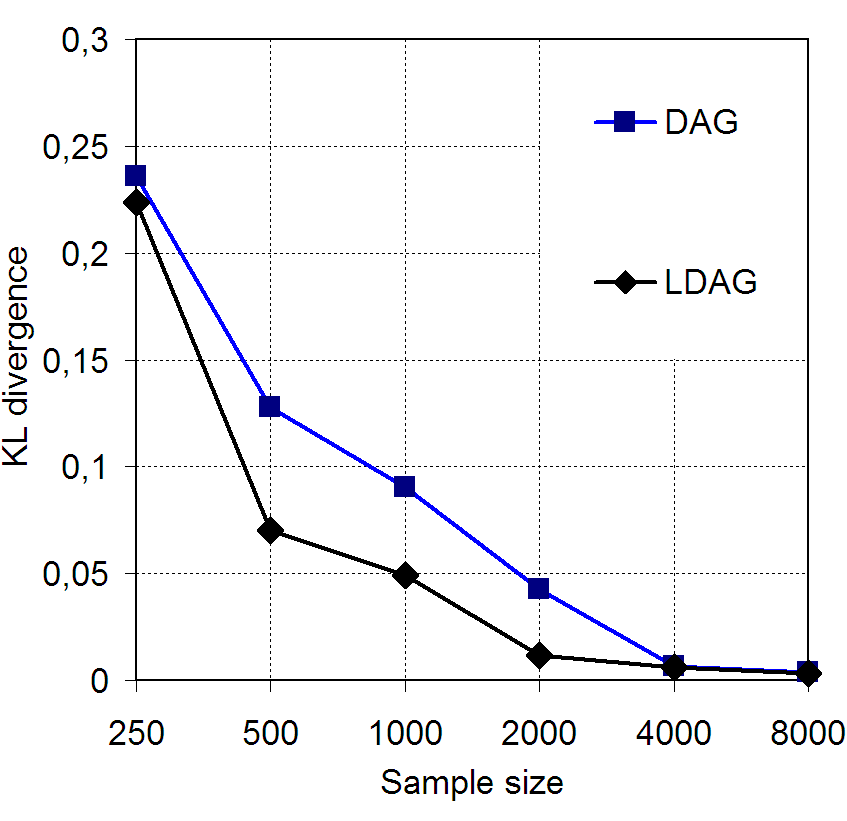}

}
\caption{Comparison of DAGs and LDAGs for different sample sizes. \label{fig:DAGvsLDAG}}
\end{centering}
\end{figure}
Figure \ref{fig:DAGvsLDAG} illustrates the difference in KL divergence between the true distribution and the approximate distributions induced by the models. The DAG curve in the figure corresponds to the $0.001$-curve from Figure \ref{fig:allCurves} and the LDAG curve corresponds to the thick black curve where the models where chosen by the initial cross-validation method. Note that the method in some cases picks the $0.001$-prior which results in a converging of the curves. We see that the LDAGs mostly outperform traditional DAGs by inducing distributions that better approximate the true distribution. This is especially clear for the medium sized samples, even when the generating model does not contain any explicit CSIs. Consequently, this seems to be the range where the models have the most to gain from adding labels. The samples are too small for discovering the true DAG structure without labels yet large enough for the structure learning to be stable even under less restrictive priors. For large enough sample sizes the two curves will eventually converge. This is a natural and inevitable phenomenon that is illuminated when investigating the structure of the underlying DAG.

If we consider the result tables (Table \ref{PropDAGdata} and \ref{PropLDAGdata}), we see that the point of convergence between the curves coincides with the sample size at which the correct underlying DAG is identified under $\kappa=0.001$, i.e. without labels. When the generating model is based on a DAG, adding labels to the correct underlying DAG will induce restrictions on the corresponding approximate distribution that is not satisfied by the true distribution. When the estimation of the parameters become stable enough, the gain from having to estimate fewer parameters cannot longer outweigh the inaccuracies of the additional restrictions. When the generating model contains explicit CSIs, the DAG curve does not overtake the LDAG curve for any of the considered sample sizes. The DAG curve will eventually catch up with the LDAG curve when $n\rightarrow\infty$ but the DAG model will, however, require some redundant parameters. Finally, the result tables also illustrate how adding labels facilitates the discovery of the true global dependence structure in the sense that LDAGs require less data to reach the correct underlying DAG compared to traditional DAGs. The flexibility of LDAGs provides an advantage over traditional DAGs in terms of structure learning, since it allows representation of more complex models with fewer parameters. However, the same flexibility may also cause overfitting if not properly regulated in the learning process.
\section{Discussion\label{sec:Discussion}}
We have further developed the idea of incorporating context-specific independence in directed graphical models by introducing a graphical representation in form of a labeled directed acyclic graph. We have shown that an LDAG is general in its representation of local CSIs as well as it is able to visualize complex dependence structures as a single entity. We also investigated properties of LDAGs in terms of model interpretability and identifiability by introducing the class of maximal regular LDAGs and the notion of CSI-equivalence.

In terms of structure learning, we have derived an LDAG-based Bayesian score and an MCMC-based search method that combines stochastic global changes with deterministic local changes. Our experimental results agree with previous research in the sense that incorporation of CSI in the learning phase improves model quality. However, we also noted that an appropriate prior must be used for optimal performance.  

An interesting extension to LDAGs could be to allow the local dependence structures to go beyond CSI in a way that can still be expressed through some form of labels. It would also be interesting to carry out a more extensive simulation study in which one could compare alternative search methods as well as compare LDAGs to other existing models.
\bibliographystyle{spbasic}
\addcontentsline{toc}{section}{\refname}\bibliography{LDAGbib}
\newpage
\section*{Appendix}
\vspace{-0.5cm}
\begin{table}[H]
\begin{centering}
\begin{tabular}{ccc}
\hline 
{\footnotesize Variable} & {\footnotesize Definition} & {\footnotesize Outcomes}\tabularnewline
\hline 
{\footnotesize $X_{1}$} & {\footnotesize Smoking} & {\footnotesize $\textrm{No}:=0,\textrm{Yes}:=1$}\tabularnewline
{\footnotesize $X_{2}$} & {\footnotesize Strenuous mental work} & {\footnotesize $\textrm{No}:=0,\textrm{Yes}:=1$}\tabularnewline
{\footnotesize $X_{3}$} & {\footnotesize Strenuous physical work} & {\footnotesize $\textrm{No}:=0,\textrm{Yes}:=1$}\tabularnewline
{\footnotesize $X_{4}$} & {\footnotesize Systolic blood pressure} & {\footnotesize $<140:=0,>140:=1$}\tabularnewline
{\footnotesize $X_{5}$} & {\footnotesize Ratio of $\beta$ and $\alpha$ lipoproteins} & {\footnotesize $<3:=0,>3:=1$}\tabularnewline
{\footnotesize $X_{6}$} & {\footnotesize Family anamnesis of coronary heart disease} & {\footnotesize $\textrm{No}:=0,\textrm{Yes}:=1$}\tabularnewline
\hline 
\end{tabular}
\caption{Description of the variables in coronary heart disease data.\label{tab:Explanation heart data}}
\end{centering}
\end{table}
\vspace{-0.5cm}
\begin{figure}[H]
\begin{center}
\includegraphics[width=\textwidth-6.5cm]{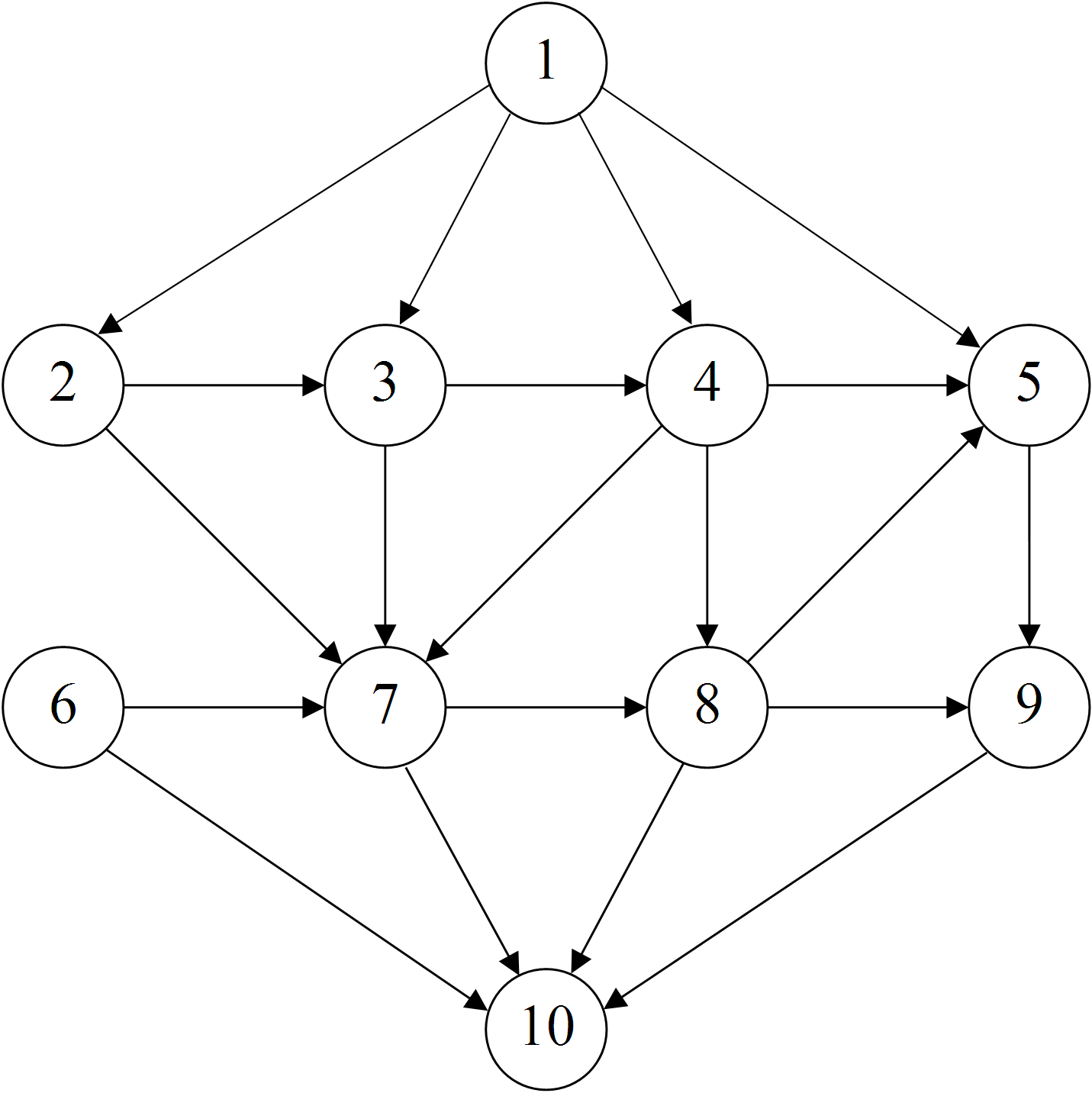}
\end{center}
\begin{center}
\begin{minipage}[t]{5cm}%
{\normalsize
\begin{eqnarray*}
\mathcal{L}_{3}:\hspace{0.5cm}	& \mathcal{L}_{(2,3)}&=\{0\}\\
\mathcal{L}_{4}:\hspace{0.5cm}	& \mathcal{L}_{(1,4)}&=\{1\}\\
\mathcal{L}_{5}:\hspace{0.5cm}	& \mathcal{L}_{(4,5)}&=\{(0,*)\}\\
			& \mathcal{L}_{(8,5)}&=\{(0,*)\}\\		
\mathcal{L}_{7}:\hspace{0.5cm}	& \mathcal{L}_{(2,7)}&=\{(1,1,0)\}\\
					& \mathcal{L}_{(3,7)}&=\{(0,1,1),(1,*,1)\}\\
					& \mathcal{L}_{(4,7)}&=\{(1,1,*)\}\\
					& \mathcal{L}_{(6,7)}&=\{(1,1,*)\}\\
\end{eqnarray*}}
\end{minipage}
\begin{minipage}[t]{5cm}%
{\normalsize
\begin{eqnarray*}
\mathcal{L}_{9}:\hspace{0.5cm}	& \mathcal{L}_{(5,9)}&=\{1\}\\
\mathcal{L}_{10}:\hspace{0.5cm}	& \mathcal{L}_{(7,10)}&=\{(1,*,*)\}\\
						& \mathcal{L}_{(8,10)}&=\{(1,*,*)\}\\
						& \mathcal{L}_{(9,10)}&=\{(1,*,*)\}
\end{eqnarray*}}
\end{minipage}
\end{center}
\caption{DAG and labels according to which the synthetic data sets were generated.\label{fig:genModels}}
\end{figure} 

\begin{table}
\begin{centering}
\begin{tabular}{cccccccc}
\hline 
{\footnotesize $n$} & {\footnotesize $\kappa$} & {\footnotesize $\log p(\mathbf{X},G_{L})$} & {\footnotesize $|E|$} & {\footnotesize $dim(\Theta_{G})-dim(\Theta_{G_{L}})$} & {\footnotesize $D_{KL}$} & {\footnotesize $\rho_{pred}(\kappa)$}\tabularnewline
\hline 
\hline 
{\footnotesize 250} &  {\textbf{\footnotesize 0.001}} & {\textbf{\footnotesize -1478.76}} & {\textbf{\footnotesize 13}} & {\textbf{\footnotesize 0}} & {\textbf{\footnotesize 0.2956}} & {\textbf{\footnotesize -146.33}}\tabularnewline
 & {\footnotesize 0.1} & {\footnotesize -1477.13} & {\footnotesize 14} & {\footnotesize 4} & {\footnotesize 0.2387} & {\footnotesize -146.96}\tabularnewline
 & {\footnotesize 0.3} & {\footnotesize -1467.68} & {\footnotesize 15} & {\footnotesize 18} & {\footnotesize 0.2873} & {\footnotesize -150.14}\tabularnewline
 & {\footnotesize 0.5} & {\footnotesize -1451.83} & {\footnotesize 20} & {\footnotesize 35} & {\footnotesize 0.5009} & {\footnotesize -157.86}\tabularnewline
\hline 
{\footnotesize 500} & {\footnotesize 0.001} & {\footnotesize -2855.23} & {\footnotesize 17} & {\footnotesize 0} & {\footnotesize 0.1755} & {\footnotesize -277.52}\tabularnewline
 & {\footnotesize 0.1} & {\footnotesize -2832.88} & {\footnotesize 19} & {\footnotesize 15} & {\footnotesize 0.1045} & {\footnotesize -276.69}\tabularnewline
{\footnotesize {*}} & {\textbf{\footnotesize 0.3}} & {\textbf{\footnotesize -2813.80}} & {\textbf{\footnotesize 20}} & {\textbf{\footnotesize 24}} & {\textbf{\footnotesize 0.1209}} & {\textbf{\footnotesize -275.13}}\tabularnewline
 & {\footnotesize 0.5} & {\footnotesize -2805.54} & {\footnotesize 21} & {\footnotesize 32} & {\footnotesize 0.1492} & {\footnotesize -280.19}\tabularnewline
\hline 
{\footnotesize 1000} & {\footnotesize 0.001} & {\footnotesize -5587.31} & {\footnotesize 18} & {\footnotesize 0} & {\footnotesize 0.0742} & {\footnotesize -546.52}\tabularnewline
 & {\footnotesize 0.1} & {\footnotesize -5569.53} & {\footnotesize 19} & {\footnotesize 13} & {\footnotesize 0.0676} & {\footnotesize -544.18}\tabularnewline
{\footnotesize {*}} & {\textbf{\footnotesize 0.3}} & {\textbf{\footnotesize -5553.73}} & {\textbf{\footnotesize 20}} & {\textbf{\footnotesize 24}}  & {\textbf{\footnotesize 0.0523}} & {\textbf{\footnotesize -544.13}}\tabularnewline
 & {\footnotesize 0.5} & {\footnotesize -5541.17} & {\footnotesize 23} & {\footnotesize 33} & {\footnotesize 0.0632} & {\footnotesize -550.40}\tabularnewline
\hline 
{\footnotesize 2000 {*}} & {\footnotesize 0.001} & {\footnotesize -11132.33} & {\footnotesize 20} & {\footnotesize 0} & {\footnotesize 0.0162} & {\footnotesize -1095.74}\tabularnewline
{\footnotesize {*}} & {\textbf{\footnotesize 0.1}} & {\textbf{\footnotesize -11097.82}} & {\textbf{\footnotesize 20}} & {\textbf{\footnotesize 18}} & {\textbf{\footnotesize 0.0215}} & {\textbf{\footnotesize -1095.68}}\tabularnewline
{\footnotesize {*}} & {\footnotesize 0.3} & {\footnotesize -11073.52} & {\footnotesize 20} & {\footnotesize 23} & {\footnotesize 0.0279} & {\footnotesize -1097.86}\tabularnewline
 & {\footnotesize 0.5} & {\footnotesize -11062.23} & {\footnotesize 24} & {\footnotesize 38} & {\footnotesize 0.0379} & {\footnotesize -1098.94}\tabularnewline
\hline 
{\footnotesize 4000 {*}} & {\footnotesize 0.001} & {\footnotesize -21988.57} & {\footnotesize 20} & {\footnotesize 0} & {\footnotesize 0.0068} & {\footnotesize -2177.83}\tabularnewline
{\footnotesize {*}} & {\textbf{\footnotesize 0.1}} & {\textbf{\footnotesize -21957.06}} & {\textbf{\footnotesize 20}} & {\textbf{\footnotesize 14}} & {\textbf{\footnotesize 0.0088}} & {\textbf{\footnotesize -2177.49}}\tabularnewline
{\footnotesize {*}} & {\footnotesize 0.3} & {\footnotesize -21940.76} & {\footnotesize 20} & {\footnotesize 15} & {\footnotesize 0.0096} & {\footnotesize -2177.77}\tabularnewline
{\footnotesize {*}} & {\footnotesize 0.5} & {\footnotesize -21929.60} & {\footnotesize 20} & {\footnotesize 17} & {\footnotesize 0.0114} & {\footnotesize -2178.32}\tabularnewline
\hline 
{\footnotesize 8000 {*}} & {\textbf{\footnotesize 0.001}} & {\textbf{\footnotesize -43822.03}} & {\textbf{\footnotesize 20}} & {\textbf{\footnotesize 0}} & {\textbf{\footnotesize 0.0034}} & {\textbf{\footnotesize -4357.81}}\tabularnewline
{\footnotesize {*}} & {\footnotesize 0.1} & {\footnotesize -43796.27} & {\footnotesize 20} & {\footnotesize 11} & {\footnotesize 0.0044} & {\footnotesize -4358.73}\tabularnewline
{\footnotesize {*}} & {\footnotesize 0.3} & {\footnotesize -43784.18} & {\footnotesize 20} & {\footnotesize 11} & {\footnotesize 0.0044} & {\footnotesize -4361.13}\tabularnewline
& {\footnotesize 0.5} & {\footnotesize -43777.35} & {\footnotesize 21} & {\footnotesize 15} & {\footnotesize 0.0051} & {\footnotesize -4361.58}\tabularnewline
\hline 
\end{tabular}
\par\end{centering}{\par}
\caption{Properties of identified LDAGs for DAG data.\label{PropDAGdata}}
\end{table}

\begin{table}
\begin{centering}
\begin{tabular}{cccccccc}
\hline 
{\footnotesize $n$} & {\footnotesize $\kappa$} & {\footnotesize $\log p(\mathbf{X},G_{L})$} & {\footnotesize $|E|$} & {\footnotesize $dim(\Theta_{G})-dim(\Theta_{G_{L}})$} & {\footnotesize $D_{KL}$} & {\footnotesize $\rho_{pred}(\kappa)$}\tabularnewline
\hline 
\hline 
{\footnotesize 250} & {\footnotesize 0.001} & {\footnotesize -1468.04} & {\footnotesize 13} & {\footnotesize 0}  & {\footnotesize 0.2357} & {\footnotesize -142.08}\tabularnewline
 & {\textbf{\footnotesize 0.1}} & {\textbf{\footnotesize -1463.54}} & {\textbf{\footnotesize 13}} & {\textbf{\footnotesize 4}} & {\textbf{\footnotesize 0.2237}} & {\textbf{\footnotesize -141.93}}\tabularnewline
 & {\footnotesize 0.3} & {\footnotesize -1451.55} & {\footnotesize 16} & {\footnotesize 15} & {\footnotesize 0.3406} & {\footnotesize -146.73}\tabularnewline
 & {\footnotesize 0.5} & {\footnotesize -1436.93} & {\footnotesize 23} & {\footnotesize 47} & {\footnotesize 0.7015} & {\footnotesize -149.02}\tabularnewline
\hline 
{\footnotesize 500} & {\footnotesize 0.001} & {\footnotesize -2943.88} & {\footnotesize 13} & {\footnotesize 0} & {\footnotesize 0.1277} & {\footnotesize -286.72}\tabularnewline
 & {\footnotesize 0.1} & {\footnotesize -2934.72} & {\footnotesize 14} & {\footnotesize 12} & {\footnotesize 0.0903} & {\footnotesize -286.34}\tabularnewline
 & {\textbf{\footnotesize 0.3}} & {\textbf{\footnotesize -2917.84}} & {\textbf{\footnotesize 17}} & {\textbf{\footnotesize 18}} & {\textbf{\footnotesize 0.0703}} & {\textbf{\footnotesize -284.12}}\tabularnewline
 & {\footnotesize 0.5} & {\footnotesize -2906.13} & {\footnotesize 21} & {\footnotesize 41} & {\footnotesize 0.1001} & {\footnotesize -284.59}\tabularnewline
\hline 
{\footnotesize 1000} & {\footnotesize 0.001} & {\footnotesize -5739.70} & {\footnotesize 16} & {\footnotesize 0} & {\footnotesize 0.0905} & {\footnotesize -565.60}\tabularnewline
 & {\footnotesize 0.1} & {\footnotesize -5725.86} & {\footnotesize 16} & {\footnotesize 15} & {\footnotesize 0.0600} & {\footnotesize -565.79}\tabularnewline
 & {\textbf{\footnotesize 0.3}} & {\textbf{\footnotesize -5705.51}} & {\textbf{\footnotesize 19}} & {\textbf{\footnotesize 25}} & {\textbf{\footnotesize 0.0490}} & {\textbf{\footnotesize -564.70}}\tabularnewline
 & {\footnotesize 0.5} & {\footnotesize -5692.74} & {\footnotesize 19} & {\footnotesize 25} & {\footnotesize 0.0490} & {\footnotesize -567.23}\tabularnewline
\hline 
{\footnotesize 2000} & {\footnotesize 0.001} & {\footnotesize -11377.83} & {\footnotesize 17} & {\footnotesize 0} & {\footnotesize 0.0428} & {\footnotesize -1122.61}\tabularnewline
 & {\footnotesize 0.1} & {\footnotesize -11328.02} & {\footnotesize 18} & {\footnotesize 18} & {\footnotesize 0.0243} & {\footnotesize -1121.34}\tabularnewline
{\footnotesize {*}} & {\textbf{\footnotesize 0.3}} & {\textbf{\footnotesize -11305.91}} & {\textbf{\footnotesize 20}} & {\textbf{\footnotesize 27}} & {\textbf{\footnotesize 0.0116}} & {\textbf{\footnotesize -1118.74}}\tabularnewline
 & {\footnotesize 0.5} & {\footnotesize -11294.84} & {\footnotesize 23} & {\footnotesize 45} & {\footnotesize 0.0171} & {\footnotesize -1119.03}\tabularnewline
\hline 
{\footnotesize 4000 {*}} & {\footnotesize 0.001} & {\footnotesize -22659.45} & {\footnotesize 20} & {\footnotesize 0} & {\footnotesize 0.0065} & {\footnotesize -2245.44}\tabularnewline
{\footnotesize {*}} & {\textbf{\footnotesize 0.1}} & {\textbf{\footnotesize -22600.54}} & {\textbf{\footnotesize 20}} & {\textbf{\footnotesize 23}} & {\textbf{\footnotesize 0.0063}} & {\textbf{\footnotesize -2243.25}}\tabularnewline
{\footnotesize {*}} & {\footnotesize 0.3} & {\footnotesize -22573.11} & {\footnotesize 20} & {\footnotesize 25} & {\footnotesize 0.0078} & {\footnotesize -2243.89}\tabularnewline
{\footnotesize {*}} & {\footnotesize 0.5} & {\footnotesize -22560.28} & {\footnotesize 20} & {\footnotesize 26} & {\footnotesize 0.0089} & {\footnotesize -2244.78}\tabularnewline
\hline 
{\footnotesize 8000 {*}} & {\footnotesize 0.001} & {\footnotesize -44862.27} & {\footnotesize 20} & {\footnotesize 0} & {\footnotesize 0.0036} & {\footnotesize -4462.29}\tabularnewline
{\footnotesize {*}} & {\textbf{\footnotesize 0.1}} & {\textbf{\footnotesize -44803.87}} & {\textbf{\footnotesize 20}} & {\textbf{\footnotesize 22}} & {\textbf{\footnotesize 0.0031}} & {\textbf{\footnotesize -4462.08}}\tabularnewline
{\footnotesize {*}} & {\footnotesize 0.3} & {\footnotesize -44778.38} & {\footnotesize 20} & {\footnotesize 24} & {\footnotesize 0.0046} & {\footnotesize -4462.43}\tabularnewline
{\footnotesize {*}} & {\footnotesize 0.5} & {\footnotesize -44766.12} & {\footnotesize 20} & {\footnotesize 24} & {\footnotesize 0.0046} & {\footnotesize -4463.07}\tabularnewline
\hline
\end{tabular}
\par\end{centering}{\par}
\caption{Properties of identified LDAGs for LDAG data.\label{PropLDAGdata}}
\end{table}

\end{document}